\documentclass[english]{article}
 \makeatletter
\let\@fnsymbol\@alph
\makeatother
\usepackage[T1]{fontenc}
\usepackage{geometry}
\usepackage{float}
\usepackage{amsmath,amsfonts}
\usepackage{amsthm}
\usepackage{amssymb}
\usepackage{graphicx}
\usepackage{url} 
\usepackage{acronym} 
\usepackage{color}
\usepackage{algorithm}
\usepackage{algorithmic}
\usepackage{authblk}

\DeclareMathOperator{\spn}{span}
\makeatletter

\floatstyle{ruled}
\newfloat{algorithm}{tbp}{loa}
\providecommand{\algorithmname}{Algorithm}
\floatname{algorithm}{\protect\algorithmname}

\@ifundefined{showcaptionsetup}{}{%
 \PassOptionsToPackage{caption=false}{subfig}}
\usepackage{subfig}
\makeatother

\newcommand{\norm}[1]{\left\lVert#1\right\rVert}
\acrodef{PCA}{Principal Component Analysis}
\acrodef{EVD}{Eigenvalue Decomposition}
\acrodef{SVD}{Singular Value Decomposition}
 \acrodef{STD}{Standard Deviation}
 
\newtheorem{theorem}{Theorem}[section]
\newtheorem{corollary}{Corollary}[theorem]

\usepackage{babel}

\begin{document}
	
%
%


\title{Mahalanonbis Distance Informed by Clustering}

\author[1]{Almog Lahav \footnote{salmogl@campus.technion.ac.il}}
\author[1]{Ronen Talmon \footnote{ronen@ee.technion.ac.il}}
\author[2]{Yuval Kluger \footnote{yuval.kluger@yale.edu}}

\affil[1]{Technion Israel Institute of Technology, Haifa 32000, Israel}
\affil[2]{Department of Pathology, Yale University School of Medicine, New Haven, CT 06520}
\date{}
\maketitle

\begin{abstract}
	{	A fundamental question in data analysis, machine learning and signal processing is how to compare between data points. The choice of the distance metric is specifically challenging for high-dimensional data sets, where the problem of meaningfulness is more prominent (e.g. the Euclidean distance between images).
		In this paper, we propose to exploit a property of high-dimensional data that is usually ignored - which is the structure stemming from the relationships between the coordinates. 	
		Specifically we show that organizing similar coordinates in clusters can be exploited for the construction of the Mahalanobis distance between samples. 
		When the observable samples are generated by a nonlinear transformation of hidden variables, the Mahalanobis distance allows the recovery of the Euclidean distances in the hidden space.
		We illustrate the advantage of our approach on a synthetic example where the discovery of clusters of correlated coordinates improves	the estimation of the principal directions of the samples. 
		Our method was applied to real data of gene expression for lung adenocarcinomas (lung cancer). By using the proposed metric we found a partition of subjects to risk groups	with a good separation between their Kaplan-Meier survival plot.
	}
	{metric learning, geometric analysis, manifold learning, intrinsic modeling, biclustering, gene expression }
\end{abstract}

\section{Introduction}
\label{sec:intro}

Recent technological advances give rise to the collection and storage of massive complex data sets.
Consider a data matrix $\mathbf{D} \in \mathbb{R}^{m\times n}$, whose columns consist of high-dimensional samples (of dimension $m$) and whose rows consist of $n$ coordinates or features. We denote the columns (samples) in $\mathbb{R}^m$ by $\left\{ c_{i}\right\} _{i=1}^{n}$, and the rows (coordinates) in $\mathbb{R}^n$ by $\left\{ r_{j}\right\} _{j=1}^{m}$.

Many analysis techniques, such as kernel methods \cite{lai2000kernel, scholkopf2001learning, liu2011kernel} and manifold learning \cite{belkin2003laplacian,coifman2006diffusion,roweis2000nonlinear,tenenbaum2000global}, are based on comparisons between samples. Such analysis can be applied to the data stored in the matrix $\mathbf{D}$, either by measuring similarities between columns or between rows. Furthermore, for high dimensional data set, this can be done in both dimensions, columns and rows.

An example for analysis that is based on distances in both dimensions  
was presented by Brazma et al. \cite{brazma2000gene} in the context of gene expression. In this case, the columns represent subjects and the rows are their gene expressions. By comparing rows, we may recover co-regulated or functionally associated genes. In addition, similarity between pair of columns indicates subjects with similar gene expression profiles.
Both comparisons (of the rows and the columns) are independent, yet, often, and in particular in this example, the columns and rows might be coupled, e.g., a group of genes are co-expressed, only within a subset of subjects which belong to a specific population group.

Existing approaches for simultaneous analysis of samples and coordinates are based on bi-clustering \cite{madeira2004biclustering,busygin2008biclustering,chi2017convex}. Chi et al. proposed in \cite{chi2017convex} a convex formulation of the bi-clustering problem based on the assumption that the matrix $\mathbf{D}$ exhibits a latent checkerboard structure. Their solution attempts to approximate this structure by alternating between convex clustering of rows and convex clustering of columns, until convergence to a global minimum. However, this method, as well as other bi-clustering techniques, is restricted by the checkerboard structure assumption. Even if this assumption is indeed satisfied, the outcome represents the data in a very specific manner: to which 2D cluster each entry $\mathbf{D_{i,j}}$ belongs.

Manifold learning methods that take into account the coupling between dimensions (e.g., rows and columns) are presented in \cite{ankenman2014geometry, mishne2016hierarchical}. These methods describe algorithms for the construction of various data structures at multiple scales. 
Yet, they merely present constructive procedures lacking definitive (optimization) criteria.

In this work we present a framework for analysis of high dimensional
data sets, where the metric on the samples (columns) is informed by an organization
of the coordinates (rows). With the new {\em informed metric} we can organize
the samples, as well as perform other analysis tasks, such as: classification, data
representation, dimensionality reduction, etc. We show that in the case where the rows have a clustering structure, i.e. the k-means cost function takes on values
which are low enough, it assists us in the estimation of the (local or global) principal directions of the columns. 
Specifically we show that clusters of rows can be exploited for the computation of the Mahalanobis distance \cite{mahalanobis1936generalized} between columns, which is based on estimation of the inverse covariance matrix. One of the important properties of the Mahalanobis distance is its invariance to linear transformations. Assuming the observable columns are a linear function, $f$, of a set of hidden variables, the above property allows the recovery of the Euclidean distances in the hidden space. 
Based on a local version of the Mahalanobis distance proposed in \cite{Singer2008}, our approach is extended to the nonlinear case, where the Euclidean distances in the hidden space can be recovered even when $f$ is nonlinear.  

Our framework is beneficial especially when there is a limited number of columns for the computation of Mahalanobis distance. Either the whole data set is small compared to the rank of the inverse covariance matrix, or
a good locality requires a small neighborhood and therefore fewer
samples are taken for the estimation of the local inverse covariance.

We illustrate the advantage of this framework with a synthetic example,
where the discovery of clusters of correlated coordinates improves the estimation of the principal directions of the samples. The global Mahalanobis distance based on the proposed principal directions allows the separation between two classes of samples, while the conventional Mahalanobis distance fails. 

The framework was applied to real data of gene expression for lung
adenocarcinomas (lung cancer). We use a data set of gene expression
for $82$ lung adenocarcinomas. Consider
the above setup, the patients were defined as the samples, and the
$200$ genes with the highest variance as the coordinates. We built
a metric on the subjects that is informed by clusters of genes. By
using this metric we found a partition of the subjects to risk groups
with a good separation between their Kaplan-Meier survival plot. Based
on the p-value measure for this partition, our method achieves better
results than other data-driven algorithms \cite{Ma2015}, and it is
comparable to competing supervised algorithms \cite{shedden2008gene},
some of them based on additional clinical information. 

The remainder of the paper is organized as follows. In Section \ref{sec:background} we review the mathematical background of our method. Sections \ref{sec:global_M} and \ref{sec:local_M} present the global and the local versions of the proposed metric, respectively. In Section \ref{sec:toy_exp} we examine the global distance on a synthetic data set. We next present the application of the proposed local Mahalanobis distance to gene expression data in Section \ref{sec:gene}. Section \ref{sec:conclusions} concludes the paper.

\section{Background}
\label{sec:background}
\subsection{Global Mahalanobis Distance and PCA}
\label{subsec:mah_pca}

Consider an unobservable $K$-dimensional random vector $x$ with zero mean and covariance matrix $\Sigma_{x}$ defined on $\mathcal{X}\subseteq{\mathbb{R}^K}$.
An $m$-dimensional accessible random vector $c$, where $m \ge K$, is given by a linear function of $x$:
\begin{equation}
\label{eq:linear_t}
c=Ax
\end{equation}
where $A$ is an unknown $m\times K$ matrix of rank $K$. The covariance matrix of $c$ is therefore given by:
\begin{equation}
\Sigma=A\Sigma_{x}A^{T}
\end{equation}
Realizations of $x$ and $c$ are denoted by $x_i$ and $c_i$, respectively.
The above model represents many real-world problems in signal processing and data analysis where the data is driven by some latent process, and the observable samples (columns) $\left\{c_{i}\right\} _{i=1}^{n}$ are a function of a set of intrinsic hidden variables $\left\{x_{i}\right\} _{i=1}^{n}$.
The Mahalanobis distance between two samples in the hidden space, $x_1, x_2 \in \mathcal{X}$, is defined as follows:
\begin{equation}
d_{M_{x}}^{2}\left(x_{1},x_{2}\right)=\left(x_{1}-x_{2}\right)^{T}\Sigma_{x}^{-1}\left(x_{1}-x_{2}\right)
\label{eq:mahalanobis}
\end{equation}
One of the important properties of this distance is its invariance
to any linear transformation of rank $K$ of the data samples in $\mathcal{X}$.
The Mahalanobis distance between any two samples $c_1, c_2$ from the random vector $c$, is defined by: 
\begin{equation}
d_{M}^{2}\left(c_{1},c_{2}\right) =\left(c_{1}-c_{2}\right)^{T}\Sigma^{\dagger}\left(c_{1}-c_{2}\right)
\label{eq:mahalanobis_c}
\end{equation}
Note that here the Mahalanobis distance is written with the
pseudo-inverse of $\Sigma$ since it is not full rank (see appendix A):
\begin{equation}
\Sigma^{\dagger}=A^{\dagger^{T}}\Sigma_{x}^{-1}A^{\dagger}.\label{eq:cov_c_inv}
\end{equation}
\begin{theorem}
	\label{mahalnobis_theorem}
	Let $c_1, c_2$ be two samples from the random vector $c$. The Mahalanobis distance between them satisfies:
	\begin{equation}
	d_{M}^{2}\left(c_{1},c_{2}\right) = d_{M_{x}}^{2}\left(x_{1},x_{2}\right)
	\label{eq:mahalanobis_theorem}
	\end{equation}
\end{theorem}

\begin{proof}
	
	Substituting \eqref{eq:cov_c_inv} into \eqref{eq:mahalanobis_c}:
	\begin{align}
	d_{M}^{2}\left(c_{1},c_{2}\right)
	= & \left(Ax_{1}-Ax_{2}\right)^{T}A^{\dagger^{T}}\Sigma_{x}^{-1}A^{\dagger}\left(Ax_{1}-Ax_{2}\right)\\ \nonumber
	= & \left(x_{1}-x_{2}\right)^{T}\Sigma_{x}^{-1}\left(x_{1}-x_{2}\right)\\ \nonumber
	= & d_{M_{x}}^{2}\left(x_{1},x_{2}\right)
	\end{align}
\end{proof}

A consequence of Theorem \ref{mahalnobis_theorem} is the statement in the next 
corollary.

\begin{corollary}
	If we assume that the coordinates of $x$ are uncorrelated
	with unit variance, i.e. $\Sigma_{x}=I$, then
	\begin{equation}
	d_{M}^{2}\left(c_{1},c_{2}\right) = \norm{x_{1}-x_{2}}_{2}^{2}.
	\end{equation}
\end{corollary}
In other words, we can recover the Euclidean distance in the hidden space by computing the Mahalanobis distance in the observed space.

According to the model assumptions, the linear transformation $A$ is unknown and therefore the computation of $d_{M}^{2}\left(c_{1},c_{2}\right)$ requires an estimation of the covariance $\Sigma$. A common estimator of $\Sigma$ is the sample covariance $\Sigma_D = \frac{1}{n-1}\mathbf{DD}^T$. Since the Mahalanobis distance requires the pseudo-inverse of the covariance matrix, we use \ac{PCA} to find the $K$ principal directions corresponding to the rank of the true covariance matrix $\Sigma$. The \ac{EVD} of the sample covariance is given by:
\begin{equation}
\widehat{\Sigma}=U_{K}\Lambda_{K}U_{K}^{T},\label{eq:conventional_cov_rank_k}
\end{equation}
where the columns of $U_{K}\in\mathbb{R}^{m\times K}$ are the $K$
principal directions of $\mathbf{D}$, and $\Lambda_{K}$ is a diagonal matrix
consisting of the variance of each of the principal directions. Hence, the estimated Mahalanobis distance using \ac{PCA} is given by:
\begin{align}
\label{eq:global_mahb_dist}
{\hat{d}}_{M}^{2}\left(c_{1},c_{2}\right) & =\left(c_{1}-c_{2}\right)^{T}\widehat{\Sigma}^{\dagger}\left(c_{1}-c_{2}\right)\\ \nonumber
& =\left(c_{1}-c_{2}\right)^{T}\left(U_{K}\Lambda_{K}^{-1}U_{K}^{T}\right)\left(c_{1}-c_{2}\right)
\end{align}
Note that since the pseudo-inverse $\widehat{\Sigma}^{\dagger}$ is semi-positive definite (and symmetric) it can be written as $\widehat{\Sigma}^{\dagger}=WW^T$, where $W=U_K\Lambda_{K}^{-1/2}$. Substituting it into \eqref{eq:global_mahb_dist} gives:
\begin{align}
\label{eq:global_mahb_dist_pca}
{\hat{d}}_{M}^{2}\left(c_{1},c_{2}\right) & =\left(c_{1}-c_{2}\right)^{T}WW^T\left(c_{1}-c_{2}\right)\\ \nonumber
& =\left(W^Tc_{1}-W^Tc_{2}\right)^{T}\left(W^Tc_{1}-W^Tc_{2}\right)\\ \nonumber
& =\norm{W^Tc_1 - W^Tc_2}_2^2
\end{align}
Therefore, the Mahalanobis distance between $c_1$ and $c_2$ is equivalent to the Euclidean distance between their projections on the subspace spanned by the $K$ principal directions. 

\subsection{The Relation Between PCA and K-means Clustering}
\label{sec:PCA_Kmeans}
Following \cite{ding2004k}, we present the
relation between k-means clustering and \ac{PCA}. Here, we consider a set of samples $\left\{ r_{j}\right\} _{j=1}^{m}$ which are the rows of the matrix $\mathbf{D}$.
Applying k-means to $\left\{ r_{j}\right\} _{j=1}^{m}$ gives rise to $K+1$ clusters of rows, which minimize the following cost function:
\begin{equation}
J_{K}=\sum_{k=1}^{K+1}\sum_{j\in G_{k}}\left(r_{j}-\mu_{k}\right)^{2}
\label{eq:kmeans_cost}
\end{equation}
where $G_{k}$ is the set of indices of samples belonging to the $k$th cluster, and $\mu_{k}$ is the centroid of the $k$th cluster. 
The clustering minimizing \eqref{eq:kmeans_cost} is denoted by $H\in\mathbb{R}^{m\times\left(K+1\right)}$,
where the $k$th column of $H$ is an indicator vector representing
which rows belong to the $k$th cluster:

\begin{equation}
\label{eq:cluster_indicator}
H_{j,k}=\begin{cases}
\begin{array}{c}
1/\sqrt{m_k}\\
0
\end{array} & \begin{array}{c}
j\in G_{k}\\
else
\end{array}\end{cases}
\end{equation}
and $m_{k}$ is the size of $k$th cluster. Recall that $U_K$ are the $K$ principal directions of $\mathbf{D}$ used in the previous section, the relation between the matrix $H$ and $U_K$ is given by the following theorem.

\begin{theorem}
	\label{HT_theorem}
	Let $U_K$ be the $K$ principal directions of $\mathbf{D}$. The relation between $U_K$ and the clustering denoted by $H$ is given by:
	\begin{equation}
	\label{eq:U_HT}
	U_K=HT
	\end{equation}
	where $T$ is $\left(K+1\right)\times K$ matrix of rank $K$.	
	
\end{theorem}

\begin{proof}
	We define the matrix:
	\begin{equation}
	\label{eq:Q_def}
	Q=\left[\begin{array}{cc}
	\frac{1}{\sqrt{m}}\mathbf{1} & U_{K}\end{array}\right]
	\end{equation}
	where $\mathbf{1}$ is an all ones vector. According to Theorem 3.3 in \cite{ding2004k} there exists an orthonormal transformation $\widehat{T}\mathbb{\in R}^{\left(K+1\right)\times\left(K+1\right)}$, which satisfies:
	\begin{equation}
	Q=H\widehat{T}\label{eq:Q_HT}
	\end{equation}
	By removing the first column in $\widehat{T}$ we define:
	\begin{equation}
	\label{eq:T_def}
	T=\widehat{T}\left[\begin{array}{c} 0\\ I_{K\times K} \end{array}\right]
	\end{equation}
	which is a matrix of rank $K$ from the definition of $\widehat{T}$, satisfying:
	\begin{equation}
	U_K=HT
	\end{equation}

\end{proof}

Theorem	\ref{HT_theorem} implies that a linear transformation maps the clustering of the rows to the principal direction of the columns. Note that the first column of $\widehat{T}$ is:
\begin{equation}
t_{1}=\left(\sqrt{m_{1}/m},\sqrt{m_{2}/m},...,\sqrt{m_{K+1}/m}\right)
\end{equation}
since it maps the clusters to the first column of $Q$ which is a constant vector:
\begin{align}
q_1(j) &= \sum_k{H_{j,k}t_1(j)} \\ \nonumber
&= \frac{1}{\sqrt{m_l}}\sqrt{\frac{m_l}{m}} \\ \nonumber
&= \frac{1}{\sqrt{m}} \quad  \forall{j=1,2,...,m} \quad, j\in G_{l}
\end{align}
Since we are interested only in the remaining columns of $Q$, which are the principal directions, we can ignore $t_{1}$ as presented in the proof of Theorem \ref{HT_theorem}.

\section{Global Mahalanobis Distance With Clustering \label{sec:Global-Mahalanobis}}
\label{sec:global_M}
\subsection{Derivation}
\label{sec:global_M_derivation}

The Mahalanobis distance presented in Section \ref{subsec:mah_pca} is defined with the pseudo-inverse of the covariance matrix $\widehat{\Sigma}^{\dagger}$. For a finite number of samples (columns) $n$ of dimension $m$ (the number of rows), the estimation of the principal directions $U_K$ becomes more difficult when the dimension $m$ increases. As a result, $\widehat{\Sigma}^{\dagger}$ is less accurate, leading to inaccurate Mahalanobis distance between samples (columns).

This conventional approach does not take into account the structure of the rows when the distance between the columns is computed. Our observation is that based on the relation between k-means and \ac{PCA}, presented in Section \ref{sec:PCA_Kmeans}, the clustering structure of the rows can be exploited for the estimation of $U_K$.

To this end, we would like to estimate the transformation $T$ which maps clusters of rows to a set of new principal directions $\widetilde{U}_{K}$.
We start with initial $K$ principal directions $U_{K}$ computed by \ac{PCA}, i.e. the $K$ eigenvectors of the sample covariance. In addition, we have the $K+1$ k-means clusters, which are ordered in the matrix $H$ as in \eqref{eq:cluster_indicator}. Since the clusters do not overlap, the rank of $H$ is $K+1$, and the pseudo-inverse $H^{\dagger}$ exists. Multiplying \eqref{eq:U_HT} by $H^{\dagger}=\left(H^{T}H\right)^{-1}H^{T}$ give the estimated transformation: 
\begin{equation}
\label{eq:T_trans}
T=\left(H^{T}H\right)^{-1}H^{T}U_K
\end{equation}
Plugging in \eqref{eq:T_trans} into \eqref{eq:U_HT} yields:
\begin{align}
\label{eq:U_tilde}
\widetilde{U}_{K} &= HT \\ \nonumber
&= H \left(H^{T}H\right)^{-1}H^{T}U_K 
\end{align}
which constitutes new principal directions by mapping $U_K$ onto an appropriate $(K+1)$-dimensional subspace.
According to \eqref{eq:cluster_indicator}, the columns of $H$ are orthonormal and therefore $H^{T}H=I$. Therefore $\widetilde{U}_{K}$
is denoted by:
\begin{equation}
\label{eq:new_pd}
\widetilde{U}_{K}=HH^{T}U_{K}
\end{equation}
To gain further insight, we examine the explicit expression
of the $(i,j)$th entry of the matrix $HH^{T}$: 
\[
\left[ HH^{T}\right] _{ij}=\begin{cases}
\begin{array}{c}
1/m_{k}\\
0
\end{array} & \begin{array}{c}
i,j\in G_{k}\\
else
\end{array}\end{cases}
\]
This implies that $\widetilde{U}_{K}$ is derived 
by averaging the coordinates of the principal directions $U_{K}$ according
to the clusters of the rows. Specifically, if the $j$th row belongs to
the cluster $G_{k}$, the $l$th principal direction is:
\begin{equation}
\label{eq:avg_pd}
\tilde{u}_{l}\left(j\right)=\frac{1}{m_{k}}\sum_{i\in G_{k}}u_{l}\left(i\right).
\end{equation}
It further implies that if the entries of each principal direction are organized according
to the clusters of the rows, the resulting principal directions are
piecewise constant. 

The new principal directions $\widetilde{U}_{k}$ can be used to establish a new estimate of the covariance matrix by:
\begin{equation}
\label{eq:new_cov}
\widetilde{\Sigma}=\widetilde{U}_{K}\Lambda_{K}\widetilde{U}_{K}^{T}
\end{equation}
Since the new principal directions are piecewise constant, $\widetilde{\Sigma}$ exhibits
a checkerboard structure. This property will be demonstrated in the sequel. 

The Mahalanobis distance between
columns, which is informed by the clustering structure of the rows, is given by:
\begin{equation}
\label{eq:mod_Mahalanobis}
\tilde{d}_{M}^{2}\left(c_{1},c_{2}\right)=\left(c_{1}-c_{2}\right)^{T}\widetilde{\Sigma}^{\dagger}\left(c_{1}-c_{2}\right)
\end{equation}
where
\begin{equation}
\label{eq:new_cov_inv}
\widetilde{\Sigma}^{\dagger}=\widetilde{U}_{K}\Lambda_{K}^{-1}\widetilde{U}_{K}^{T}
\end{equation}

{ \color{black}
	\subsection{Optimal Solution}   \label{subsec:projected_subgradient}
	The computation of the principal directions $\widetilde{U}_{K}$ proposed in Section \ref{sec:global_M_derivation} is composed of two steps. The first step is finding the optimal principal directions ${U}_{K}$ according to the standard \ac{PCA} optimization criterion, i.e., minimizing the reconstruction error (or maximizing the variance). The second step appears in \eqref{eq:new_pd}, where the principal directions ${U}_{K}$ are mapped onto the $(K+1)$-dimensional subspace defined by $H$, which we denote by $\mathcal{S}_H$. 
	This two-step procedure only guarantees that the obtained principal directions $\left\{\tilde{u}_i\right\}_{i=1}^K$ are in the subspace $\mathcal{S}_H$,
	yet there is no guarantee that they attain the minimal reconstruction error among all sets of $K$ vectors in $\mathcal{S}_H$.
	
	To guarantee local optimality, we propose a new constrained PCA problem formulation.
	According to \eqref{eq:new_pd} and the definition of the matrix $H$ in \eqref{eq:cluster_indicator}, the subspace $\mathcal{S}_H$ is spanned by the columns of the matrix $H$, which are the indicator vectors of each cluster:
	\begin{equation}
	\label{eq:cluster_indicator_vector}
	h_k(j)=\begin{cases}
	\begin{array}{c}
	1/\sqrt{m_k}\\
	0
	\end{array} & \begin{array}{c}
	j\in G_{k}\\
	else
	\end{array}\end{cases}
	\end{equation} 
	namely, $\mathcal{S}_H=\spn\{h_1,h_2,...,h_K,h_{K+1}\}$. Let $R\left(\widetilde{U}_K\right)$ be the optimization criterion of \ac{PCA} based on the reconstruction error:
	\begin{equation}
	\label{eq:opt_pca}
	R\left(\widetilde{U}_K\right) = E\left\{\|c-\widetilde{U}_K \widetilde{U}_K^T c\|_2^2\right\}
	\end{equation} 
	We define the following constrained optimization problem:
	\begin{equation}
	\label{eq:opt_pca_constr}
	\begin{aligned}
	& \underset{\widetilde{U}_K}{\text{minimize}}
	& & R\left(\widetilde{U}_K\right) \\
	& \text{subject to}
	& & \tilde{u}_i \in \mathcal{S}_H , \; i = 1, \ldots, K
	\end{aligned}
	\end{equation}
	where we require that the principal directions $\left\{\tilde{u}_i\right\}_{i=1}^K$,
	which are the columns of $\widetilde{U}_K$, attain minimal reconstruction error and are in the subspace $\mathcal{S}_H$.
	
	This problem can be solved iteratively by the gradient  projection algorithm \cite{calamai1987projected}:\\ 
	\begin{equation}
	\label{eq:subGrad}
	\begin{aligned}
	& \text{for $l = 1,2,3,...$} \\
	& \qquad \widetilde{U}_K \leftarrow \widetilde{U}_K - \alpha_l \nabla R\left(\widetilde{U}_K\right) \\
	& \qquad \widetilde{U}_K \leftarrow HH^T \widetilde{U}_K
	\end{aligned}
	\end{equation}

	In the first stage of each iteration in \eqref{eq:subGrad} we take a step of size $\alpha_l$ in the negative direction of the gradient of $R$, which is explicitly given by:
	\begin{equation}
	\label{eq:grad_exp}
	\nabla R(\widetilde{U}_K) =-2\left(\left(I-\widetilde{U}_K \widetilde{U}_K^T\right)\Sigma + \Sigma \left(I-\widetilde{U}_K \widetilde{U}_K^T\right)\right)\widetilde{U}_K
	\end{equation}
	In the second stage we project $\left\{
	{\tilde{u}}_i\right\}_{i=1}^K$ on $\mathcal{S}_H$. In \cite{calamai1987projected}, it was shown that this method is guaranteed to converge to a stationary point.}


We conclude this section with three remarks.
First, note that if we initialize \eqref{eq:subGrad} with the principal directions $U_K$ computed by standard \ac{PCA}, the first stage does not change $U_K$ since it is the minimum point of $R$ by definition, and therefore $\nabla R(U_K) = 0$. In the second stage we get $\widetilde{U}_K = HH^T U_K$. Hence, the computation of $\widetilde{U}_K$ derived in Section \ref{sec:global_M_derivation} is equivalent to the first iteration of the optimal solution in \eqref{eq:subGrad}. Similar to the description, which appears after \eqref{eq:avg_pd}, our interpretation for \eqref{eq:subGrad} is that the optimal principal directions $\widetilde{U}_K$ are obtained by averaging the coordinates of $\widetilde{U}_K$ according to the clustering of the rows of $\mathbf{D}$ after {\em each step} toward the minimum of $R$.

Second, after the last iteration in \eqref{eq:subGrad}, the obtained principal directions are not necessarily unit vectors. 
Therefore the computation of the pseudo-inverse of the covariance matrix in \eqref{eq:new_cov_inv}, requires a normalization of each of the obtained principal directions $\widetilde{U}_K$.
Based on $\widetilde{U}_K$, the entire construction procedure of the informed distance in \eqref{eq:mod_Mahalanobis}
is summarized in Algorithm \ref{alg:Mahalanobis-distance-informed}.

Third, we use an empirical stopping criterion; the iterative algorithm stops if the update of the principal directions (the gradient) becomes smaller than a threshold, which is determined empirically in each experiment.
\begin{algorithm}
	
	\textbf{input}: Data matrix $\mathbf{D}\in\mathbb{R}^{m\times n}$
	\begin{enumerate}
		\item Start with rows $\left\{ r_{j}\right\} _{j=1}^{m}$. Apply k-means
		clustering and get $H$
		\item Find $K$ principal directions, $U_{K}=\left[u_{1}\,u_{2}\,...\,u_{K}\right]$, by taking the
		$K$ eigenvectors of $\Sigma_D$ with the largest eigenvalues $\left\{ \lambda_{i}\right\} _{i=1}^{K}$
		
		\item Start with $\widetilde{U}_K = U_K$. Compute new principal directions:
		\begin{equation}
		\begin{aligned}
		& \text{for $l = 1,2,3,...$} \\
		& \qquad \widetilde{U}_K \leftarrow \widetilde{U}_K - \alpha_l \nabla R\left(\widetilde{U}_K\right) \\ \nonumber
		& \qquad \widetilde{U}_K \leftarrow HH^T \widetilde{U}_K
		\end{aligned}
		\end{equation}
		where:
		\begin{equation}
		\nabla R(\widetilde{U}_K) =-2\left(\left(I-\widetilde{U}_K \widetilde{U}_K^T\right)\Sigma_D + \Sigma_D \left(I-\widetilde{U}_K \widetilde{U}_K^T\right)\right)\widetilde{U}_K \nonumber
		\end{equation}
		\item Normalize the principal directions $\widetilde{U}_K$
		\item Define $\Lambda_{K}^{-1}=diag\left\{ \lambda_{1}^{-1},\lambda_{2}^{-1},...,\lambda_{K}^{-1}\right\} $,
		and the distance between columns is given by:
		\[
		\tilde{d}_{M}^{2}\left(c_{1},c_{2}\right)=\left(c_{1}-c_{2}\right)^{T}\widetilde{U}_{K}\Lambda_{K}^{-1}\widetilde{U}_{K}^{T}\left(c_{1}-c_{2}\right)
		\]
	\end{enumerate}
	\caption{Global Mahalanobis Distance Informed by Clustering \label{alg:Mahalanobis-distance-informed}}
\end{algorithm}

\subsection{Example: Recovering Distances in a Hidden Space  \label{subsec:Global-exp}}

The following example demonstrates the property of Mahalanobis distance
presented in Section \ref{subsec:mah_pca}: the invariance to a linear transformation and the ability to recover distances in a hidden space.
We assume that samples in a hidden space $\mathcal{X}$ are distributed uniformly
in the unit square $x\sim U\left[0,1\right]\times\left[0,1\right]$.
The observable data samples are given by a linear function of $x$:
\begin{equation*}
c=Ax,
\end{equation*}
The matrix $A\in\mathbb{R^{\text{1200\ensuremath{\times}2}}}$ consists of three sub-matrices:
\begin{eqnarray*}
	A & = & \left[\begin{array}{c}
		A^{(1)}\\
		A^{(2)}\\
		A^{(3)}
	\end{array}\right]
\end{eqnarray*}
where $A^{(i)}\in\mathbb{R}^{400\times2}$. The rows of each of the sub-matrices $A^{(i)}$
are realizations of Gaussian random vectors with covariance $\Sigma_A=0.01 I$
and the following respective means:
\begin{eqnarray*}
	\mu_{1} & = & \left[\begin{array}{c}
		1/3\\
		1
	\end{array}\right]\,\,\,,\,\,\,\mu_{2}=\left[\begin{array}{c}
		1/3\\
		-1
	\end{array}\right]\,\,\,,\,\,\,\mu_{3}=\left[\begin{array}{c}
		-1\\
		-3
	\end{array}\right]
\end{eqnarray*}
Note that once realized, the matrix $A$ is maintained fixed through out the experiment. The structure of the matrix $A$ was chosen such that coordinates of the observable samples $c_i$ belong to three different clusters.
Figure \ref{fig:exp1_data_C} presents $\mathbf{D} \in \mathbb{R}^{1200 \times 500}$ which consists of $500$ realizations of $c$. It can be observed that the rows of $\mathbf{D}$ have a clear structure of three clusters.\\
\begin{figure}
	\centering{}\includegraphics[width=0.33\paperwidth,height=0.22\paperheight]{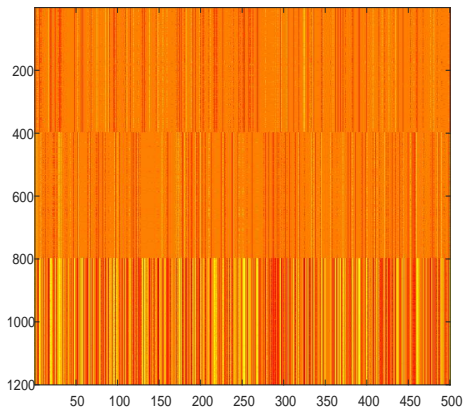}\caption{The matrix $\mathbf{D} \in \mathbb{R}^{1200\times 500}$. Each column of $\mathbf{D}$ is a realization of the random variable $c$. The rows have a clear structure of three clusters. \label{fig:exp1_data_C}}
\end{figure}
The conventional Mahalanobis distance and the distance proposed in Algorithm \ref{alg:Mahalanobis-distance-informed}, $\hat{d}_M$ and  $\tilde{d}_{M}$,
are computed from $50$ columns of $\mathbf{D}$ and compared to the Euclidean distance in the hidden space $\mathcal{X}$. As can be seen in Fig. \ref{fig:exp1_compare_dist}, both methods demonstrate accurate recovery of the Euclidean distance in $\mathcal{X}$. Particularly, both attain $0.99$ correlation with the Euclidean distance between the unobservable data samples.
\begin{figure}
	\begin{centering}
		\subfloat[]{
			\centering{}\includegraphics[width=0.35\paperwidth,height=0.2\paperheight]{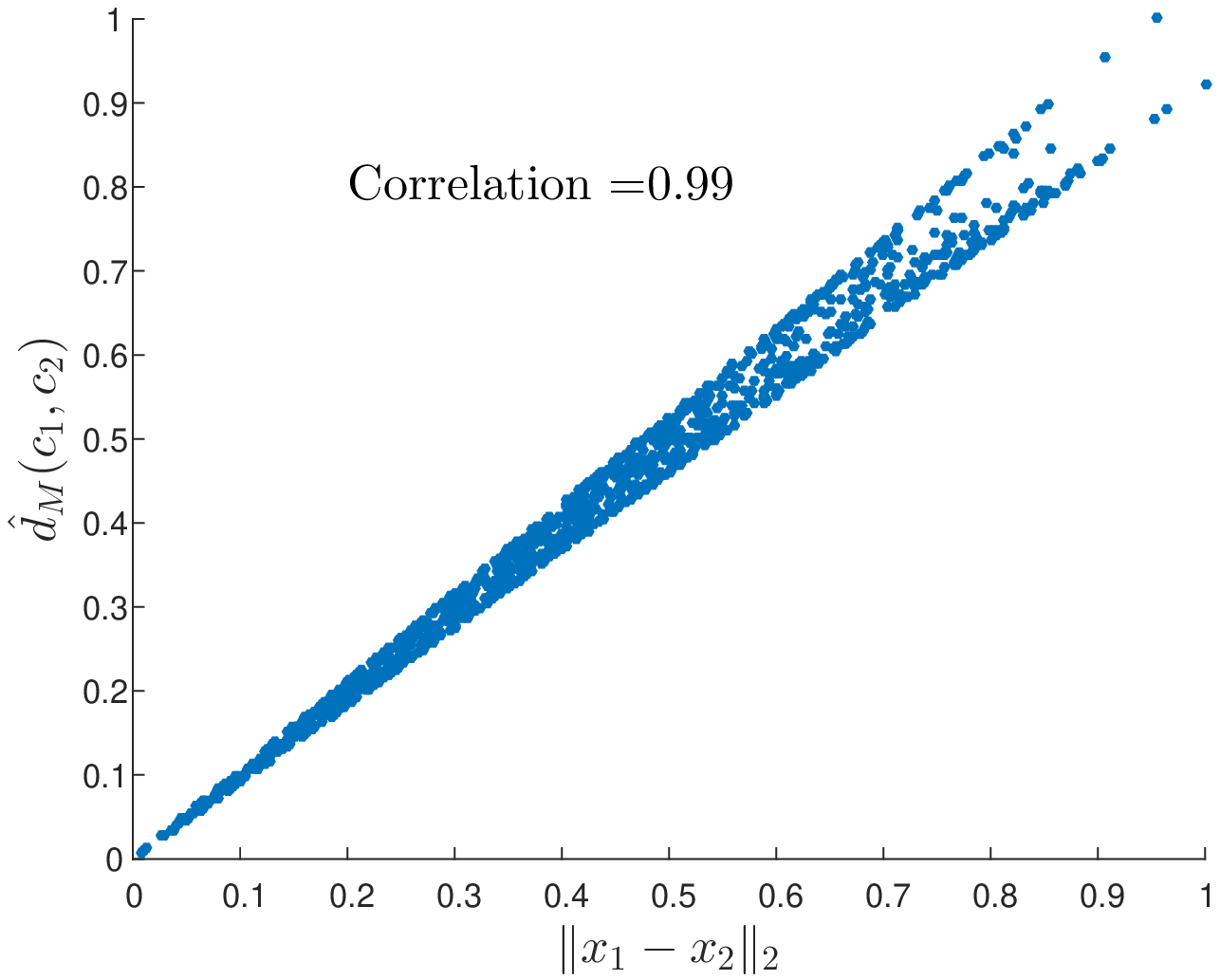}}
		\subfloat[]{
			\centering{}\includegraphics[width=0.35\paperwidth,height=0.2\paperheight]{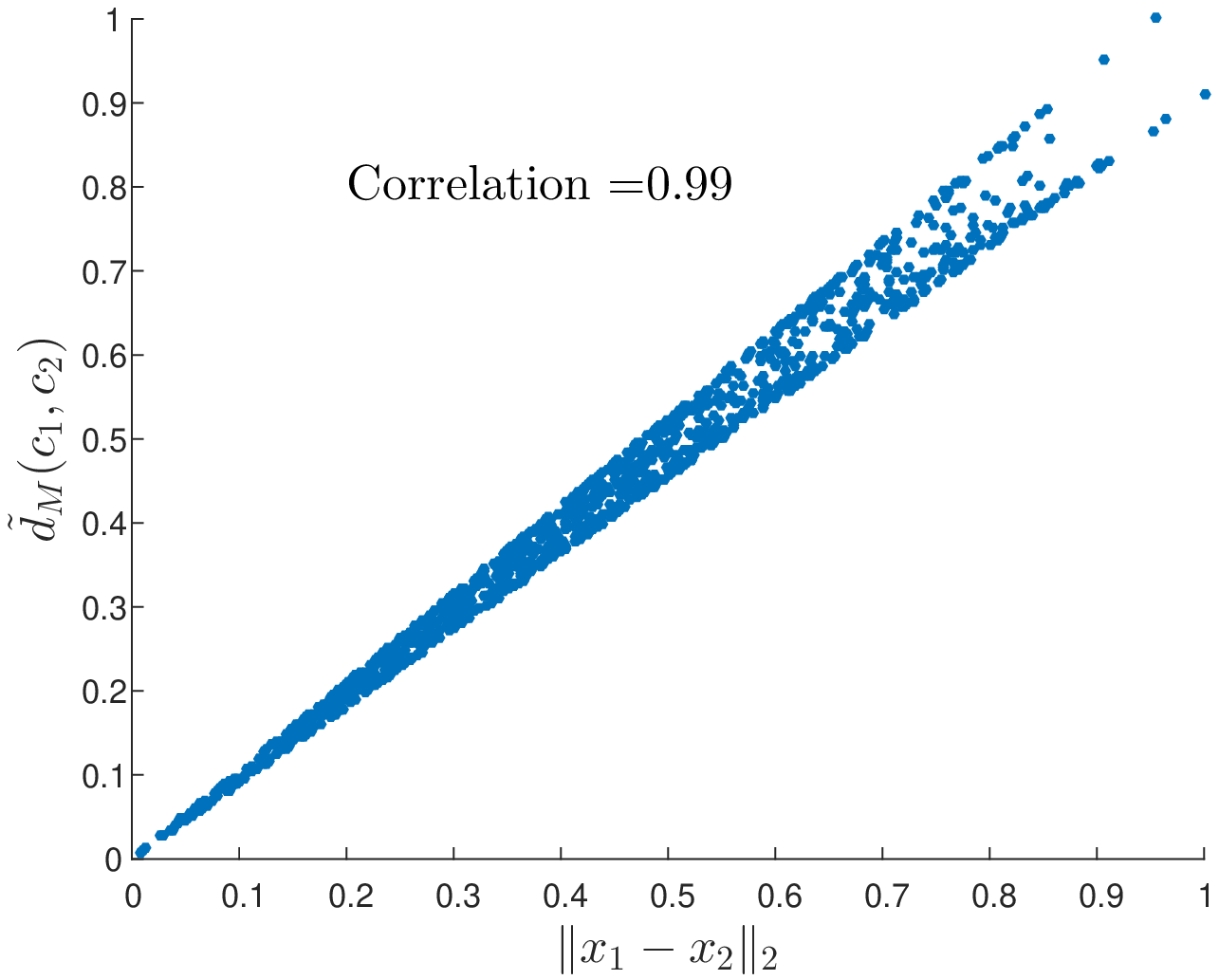}}
		\par\end{centering}
	\centering{}\caption{Comparison between the Euclidean distance in the hidden space $\mathcal{X}$ and the Mahalanobis distance computed by: (a) the conventional approach using \ac{PCA} ($\hat{d}_M$), (b) the proposed approach ($\tilde{d}_{M}$). Both obtain correlation of 0.99 with the original distance. \label{fig:exp1_compare_dist}}
\end{figure}

To examine the recovered distances in a different way, we find a parametrization of $x$ using an algorithm that uses a kernel built from small distances called diffusion maps \cite{Coifman2006}.
Following is a brief review of the main steps of diffusion maps. 
We compute an affinity matrix with the following Gaussian kernel:
\begin{equation}
\label{eq:exp1_kernel}
W_{i,j}=e^{-d^{2}\left(c_{i},c_{j}\right)/\varepsilon}
\end{equation}
Note that $d\left(c_{i},c_{j}\right)$ is any distance between the samples $c_1$ and $c_2$. The affinity matrix is normalized by the diagonal matrix $S$:
\begin{equation}
P=S^{-1}W
\end{equation}
where:
\begin{equation}
S_{i,i}=\sum_{j=1}^{n}W_{i,j}
\end{equation}
The diffusion maps embedding of the point $c_{i}$ is defined by:
\begin{equation}
\label{eq:diff_maps}
c_{i}\rightarrow\tilde{c}_{i}=\left(\lambda_{1}\phi_{1}\left(i\right),\lambda_{2}\phi_{2}\left(i\right)\right),
\end{equation}
where $\phi_{1}$ and $\phi_{2}$ are the eigenvectors of $P$ corresponding to the largest eigenvalues $\lambda_{1}$ and $\lambda_{2}$, respectively (ignoring the trivial eigenvalue $\lambda_{0}=1$).
It was proven in \cite{Coifman2006} that the operator:
\begin{equation}
L=\frac{I-P}{\varepsilon}
\end{equation}
is an approximation of a continuous Laplace operator defined in the unit square, when $x$ is distributed uniformly, as in this example. Therefore, the parametrization of $x=\left(x\left(1\right),x\left(2\right)\right)$ in \eqref{eq:diff_maps} is approximated by the two first eigenfunctions
of the Laplace operator:
\begin{eqnarray*}
	\phi_{1}\left(x\left(1\right),x\left(2\right)\right) & = & \cos\left(\pi x\left(1\right)\right)\\
	\phi_{2}\left(x\left(1\right),x\left(2\right)\right) & = & \cos\left(\pi x\left(2\right)\right)
\end{eqnarray*}
Figure \ref{fig:exp1_data} presents the data points in the original space
$\mathcal{X}$ colored by the first eigenvector. We compute the affinity matrix defined in 
\eqref{eq:exp1_kernel} for the following distances: Euclidean, $\hat{d}_{M}$ and $\tilde{d}_{M}$.
The embedding obtained by each of the distances is presented in Fig.
\ref{fig:exp1_embd_euc}, \ref{fig:exp1_embd_mah} and \ref{fig:exp1_embd_mah_inf},
where the points are colored according to $x\left(1\right)$. It can be observed that the embeddings computed by both
$\hat{d}_{M}$ and $\tilde{d}_{M}$ are accurate representations for the data in the hidden
space, and the columns are mapped into a shape homeomorphic to the unit square. Conversely, using the Euclidean distance between the columns, as demonstrated in Figure \ref{fig:exp1_embd_euc}, gave rise to a different embedding, which does not represent the hidden space. 
This example demonstrates the ability to recover distances in a hidden space by computing the Mahalanobis distance, either with the conventional approach (using $\hat{d}_M$) or with the one proposed in Algorithm \ref{alg:Mahalanobis-distance-informed} (using $\tilde{d}_M$). The advantage of the proposed distance $\tilde{d}_{M}$ will be presented in Section \ref{sec:toy_exp}.

\begin{figure}
	\begin{centering}
		\subfloat[\label{fig:exp1_data}]{\centering{}\includegraphics[width=0.3\paperwidth,height=0.15\paperheight]{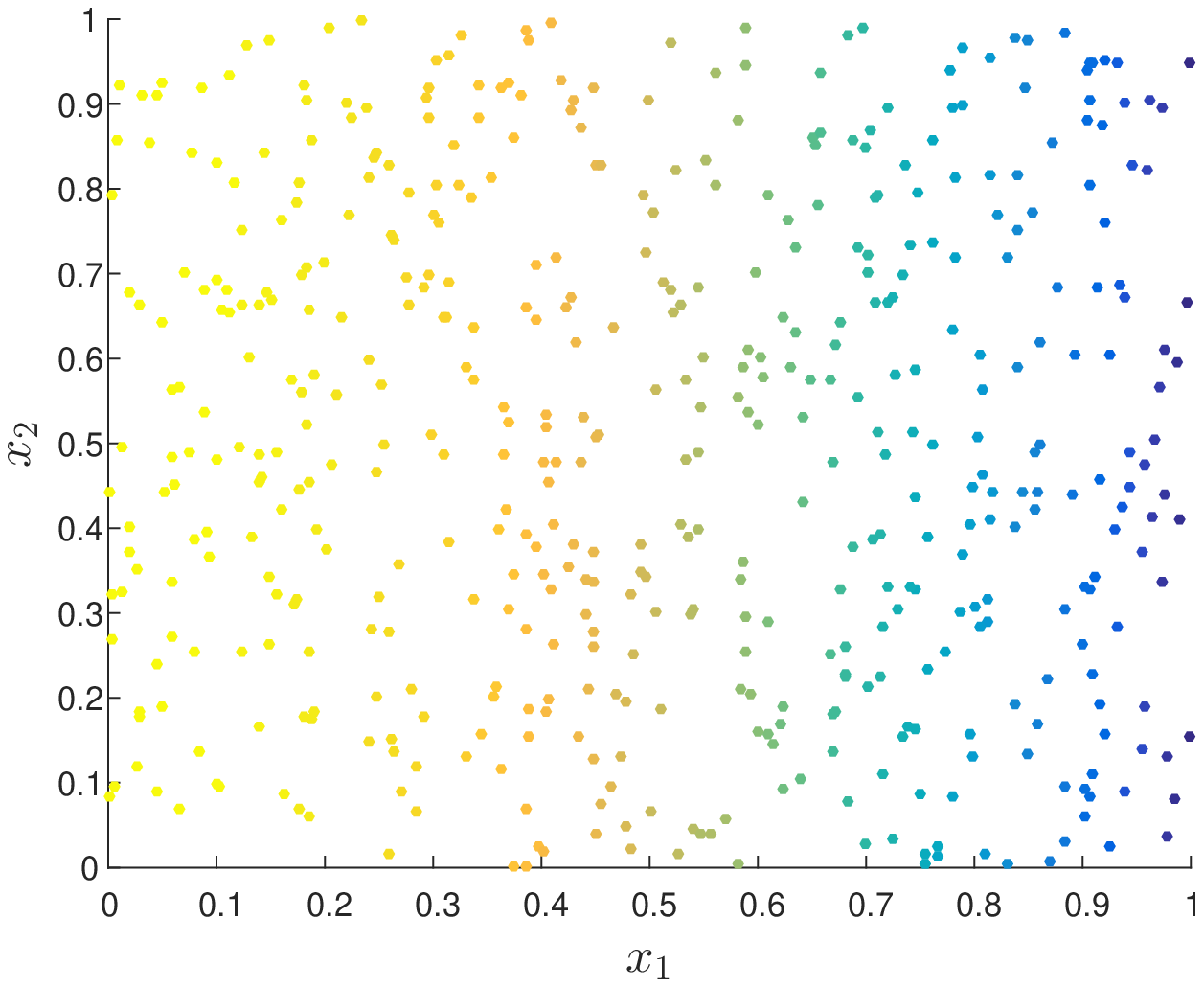}}
		\subfloat[\label{fig:exp1_embd_euc}]{\centering{}\includegraphics[width=0.3\paperwidth,height=0.15\paperheight]{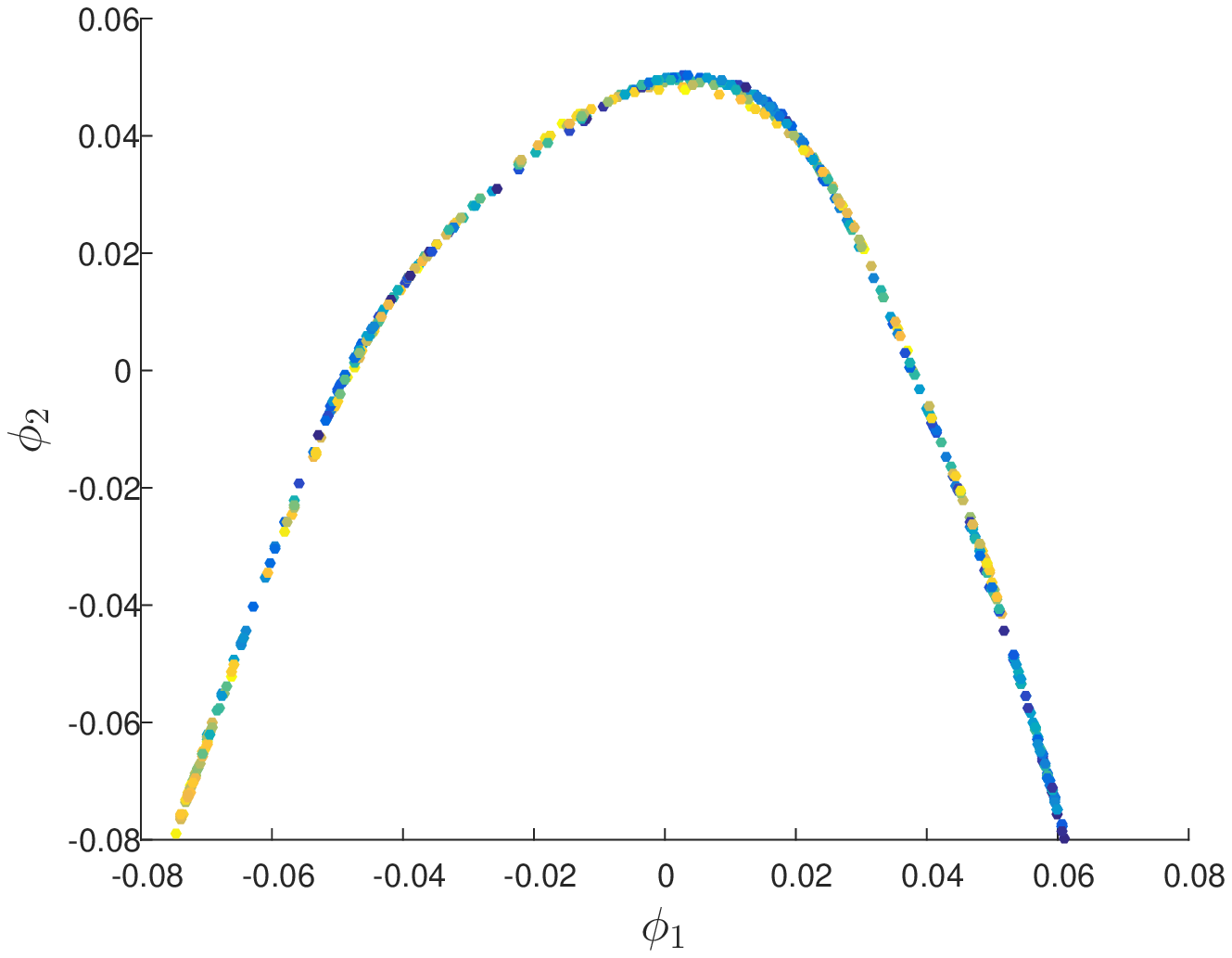}}
		\par\end{centering}
	\centering{}\subfloat[\label{fig:exp1_embd_mah}]{\centering{}\includegraphics[width=0.3\paperwidth,height=0.15\paperheight]{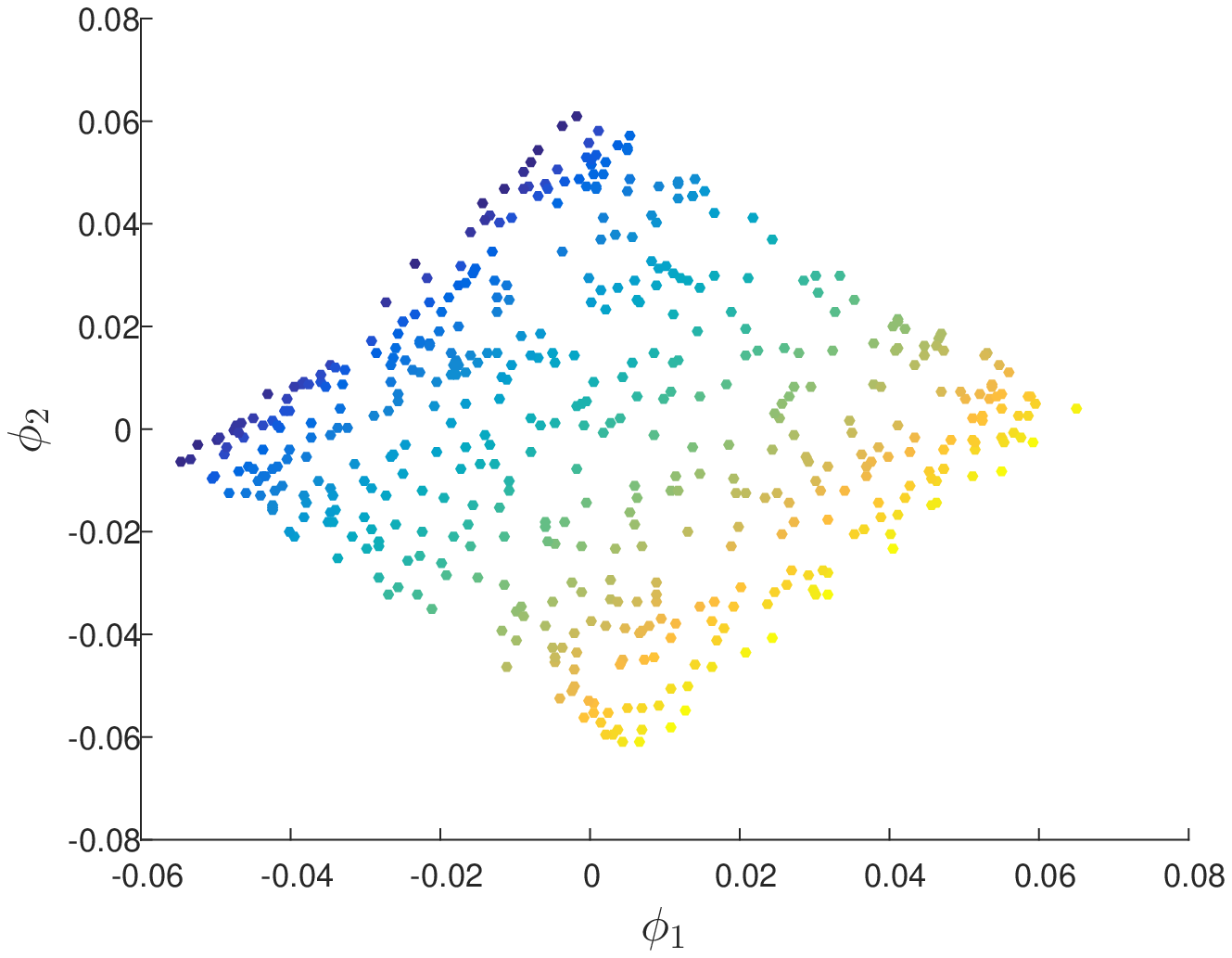}}
	\subfloat[\label{fig:exp1_embd_mah_inf}]{\centering{}\includegraphics[width=0.3\paperwidth,height=0.15\paperheight]{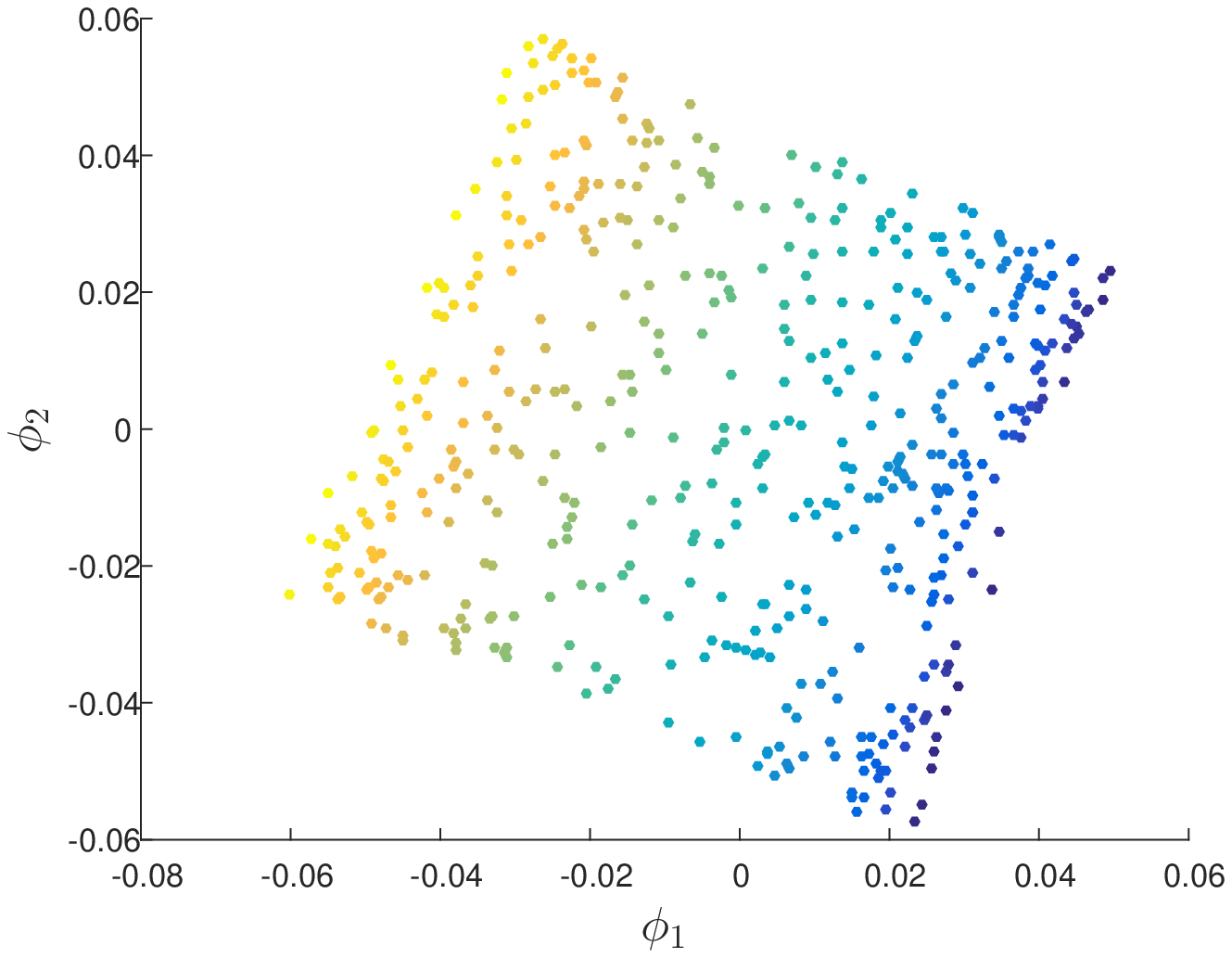}}
	\caption{(a) Data points distribute uniformly in the hidden space $\mathcal{X}$  colored by the first eigenfunction $\phi_1$. The embedding of $\left\{ c_{i}\right\} _{i=1}^{n}$ computed by: Euclidean distance, Mahalanobis distance $\hat{d}_{M}$ and the informed distance $\tilde{d}_{M}$ are presented in (b), (c) and (d).}
\end{figure}

\section{Local Mahalanobis Distance With Clustering}
\label{sec:local_M}
\subsection{Local Mahalanobis Distance}
\label{sec:local_M_pca}
The advantage of using Mahalanobis distance holds as longs as the data arises from the model \eqref{eq:linear_t}, where the random vector $c$ is a linear function of a hidden variable $x$. In this section, based on \cite{Singer2008}, we present an extension to the nonlinear case, i.e:
\begin{equation}
c=f\left(x\right)
\end{equation}
where $f$ is a non-linear function. For simplicity, in this section, we assume that $\Sigma_x=I$. This assumption is common practice, for example, in probabilistic \ac{PCA} \cite{tipping1999probabilistic}. Under this assumption, the Mahalanobis distance between points in the hidden space $d_{M_x}$ defined in \eqref{eq:mahalanobis} is Euclidean. 

To deal with the nonlinearity introduced by $f$, an extension of the Mahalanobis distance $d_{M}$ in \eqref{eq:mahalanobis_c} was presented in \cite{Singer2008}. 
Since $f$ is a non-linear function the covariance of the random variable $c$ around each point is different. We denote the local covariance around the column $c_i$ by $\Sigma_i$.
To estimate the local covariance we define the matrix $\mathbf{D}^{(i)}$ which consists of all the neighbors of the column $c_i$, namely, all the columns $c_j$ in $\mathbf{D}$ such that $\ \|c_i - c_j\| < \delta$ for some $\delta>0$.
Let $U_{d}^{\left(i\right)}$ be the principal directions of $\mathbf{D}^{(i)}$ (obtained from the \ac{EVD} of $\mathbf{D}^{(i)}$), and ${\Lambda_{d}^{(i)}}$ is a diagonal matrix
containing the variance in each of the principal directions. The estimated local covariance is given by:
\begin{equation}
\label{eq:est_local}
\widehat{\Sigma}_i = U_{d}^{(i)}{\Lambda_{d}^{(i)}}{U_{d}^{(i)}}^{T}
\end{equation}
For any two columns $c_1$ and $c_2$, a modified (local) Mahalanobis distance is defined by:
\begin{equation}
\label{eq:d_LM}
d_{LM}^2(c_1,c_2) = \frac{1}{2}\left(c_{1}-c_{2}\right)^{T}\left(\widehat{\Sigma}_1^{\dagger} + \widehat{\Sigma}_2^{\dagger}\right)\left(c_{1}-c_{2}\right)
\end{equation}
%

If $f$ is invertible, by denoting $g=f^{-1}$ and the two hidden samples in $\mathcal{X}$ corresponding to $c_1$ and $c_2$, namely, $x_1=g(c_1)$ and $x_2=g(c_2)$, the Euclidean distance between $x_{1}$ and $x_{2}$ is given by:
\begin{equation}
\label{eq:x1x2dist}
\|x_{1}-x_{2}\|^{2}=\frac{1}{2}\left(c_{1}-c_{2}\right)^{T}\left(\left(J_1J_1^{T}\right)^{-1}+\left(J_2J_2^{T}\right)^{-1}\right)\left(c_{1}-c_{2}\right)+O\left(\|c_{1}-c_{2}\|^{4}\right)
\end{equation}
where $J_i$ is the Jacobian of $f$ at $x_i$, $i=\{1,2\}$.
It was proven in \cite{Singer2008} that under the assumption that $f$ is a smooth map between two smooth manifolds: 
\begin{equation}
f:\,\,\,\mathcal{M}_{x}\mapsto\mathcal{M}_{c}
\end{equation}
the local covariance $\Sigma_{i}$ is an approximation of $J_iJ_i^{T}$. Therefore, by substituting $\Sigma_{i}=J_iJ_i^{T}$ in \eqref{eq:x1x2dist}, the Euclidean distance between points $x_1$ and $x_2$ in the hidden space is given by: 
\begin{equation}
\label{eq:d_local_M}
\|x_{1}-x_{2}\|^{2}=\frac{1}{2}\left(c_{1}-c_{2}\right)^{T}\left(\Sigma_{1}^{-1}+\Sigma_{2}^{-1}\right)\left(c_{1}-c_{2}\right)+O\left(\|c_{1}-c_{2}\|^{4}\right)
\end{equation}

If the dimension of the manifold $\mathcal{M}_c$ is $d$, the data samples close to $c_{i}$ approximately lie on a $d$-dimensional
tangent space $T_{c_{i}}\mathcal{M}_{c}$. Hence, $\Sigma_{i}$ is not a full
rank matrix, so that the inverse of the local covariance matrix in \eqref{eq:d_local_M} is replaced by its pseudo-inverse:
\begin{align}
\label{eq:pinv_local}
\|x_{1}-x_{2}\|^{2} &\approx \frac{1}{2}\left(c_{1}-c_{2}\right)^{T}\left({\Sigma}_1^{\dagger} + {\Sigma}_2^{\dagger}\right)\left(c_{1}-c_{2}\right)
\end{align}
%

Using the pseudo-inverse of the estimated covariance in \eqref{eq:est_local}, we recast \eqref{eq:pinv_local} as:
\begin{align}
\label{eq:d_M_pca}
\|x_{1}-x_{2}\|^{2} &\approx \frac{1}{2}\left(c_{1}-c_{2}\right)^{T}\left(\widehat{\Sigma}_1^{\dagger} + \widehat{\Sigma}_2^{\dagger}\right)\left(c_{1}-c_{2}\right)  \\ \nonumber
&= d_{LM}^2(c_1,c_2)
\end{align}
In other words, the local Mahalanobis distance between two columns $c_1$ and $c_2$ is an approximation of the Euclidean distance between the hidden points $x_1$ and $x_2$ in the hidden space $\mathcal{X}$.
This result implies that although the data is generated from a nonlinear transformation of the hidden samples, the distance between these samples can be computed by the modified Mahalanobis distance using the pseudo-inverse of the local covariance matrices. Based on this distance we can, for example, find a representation of the samples in the hidden space, as has been done in \cite{Singer2008,talmon2013empirical}.

\subsection{Informed Distance Derivation}

In the linear case presented in Section \ref{subsec:mah_pca}, all the available samples are assumed to be realizations from the same statistical model. Conversely, in the nonlinear case, each sample may correspond to different statistical model (each column $c_i$ corresponds to a different covariance matrix $\Sigma_{i}$).
Consequently, we consider neighborhoods of similar columns as coming from the same (local) statistical model and use them in order to estimate the local covariance matrices. 
This implies the following trade-off: on the one hand, when we use larger neighborhoods, more columns are available for the covariance estimation, thereby yielding a more accurate covariance estimation on the expense of locality. On the other hand, when we use smaller neighborhoods, fewer columns are available for the estimation, so that the covariance estimation is more local, but suffers from a larger error.

This problem can be seen as the well-known bias-variance trade-off and it is illustrated in Fig. \ref{fig:cov_on_manifold}. 
The covariance matrix $\Sigma_2$ required for the computation of the distance $d_{LM}^2(c_1,c_2)$ is estimated using different size of neighborhoods: in Fig. \ref{fig:manifold_2cov} the choice of small neighborhood respect the manifold but consists of a small number of points. In Fig. \ref{fig:manifold_2cov}, due to the manifold curvature, large neighborhood with more samples, do not respect the manifold structure.
%

\begin{figure}
	\begin{centering}
		\subfloat[\label{fig:manifold_2cov}]{\centering{}\includegraphics[width=0.3\paperwidth,height=0.17\paperheight]{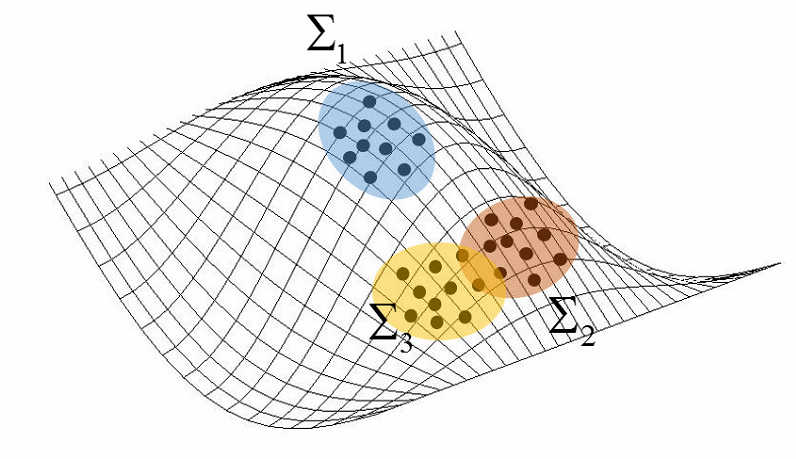}}
		\subfloat[\label{fig:manifold_1cov}]{
			\centering{}\includegraphics[width=0.3\paperwidth,height=0.17\paperheight]{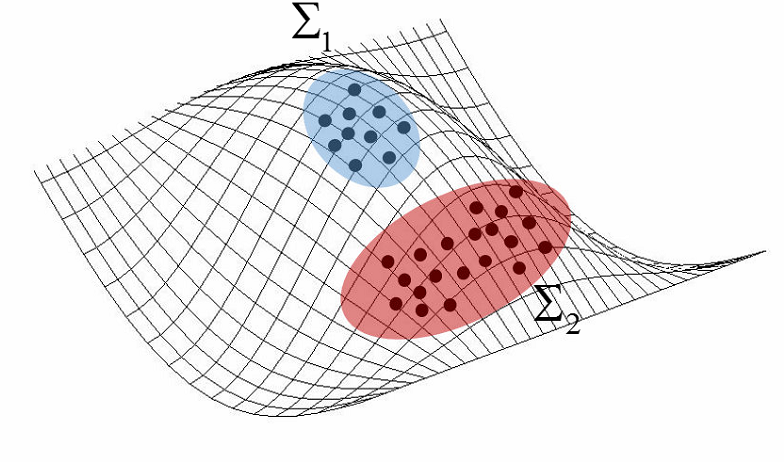}}
		\par\end{centering}
	\centering{}\caption{An illustration of the manifold $\mathcal{M}_c$, demonstrates the trade-off between locality (a) and the large number of columns required for the estimation of a local inverse covariance (b).  \label{fig:cov_on_manifold}}
\end{figure}

Considering the above trade-off, we propose a different construction of the local Mahalanobis distance based on the concept of the informed distance presented in Section \ref{sec:Global-Mahalanobis}.
%

Let $H^{(i)}$ be the clustering matrix obtained by the k-means of the rows of $\mathbf{D}^{(i)}$ defined in Section \ref{sec:local_M_pca}. 
Similarly to \eqref{eq:subGrad}, we start with $\widetilde{U}^{(i)}_K={U}^{(i)}_K$, and compute the new principal directions by:
\begin{equation}
\label{eq:subGraLocal}
\begin{aligned}
& \text{for $l = 1,2,3,...$} \\
& \qquad \widetilde{U}^{(i)}_K \leftarrow \widetilde{U}^{(i)}_K - \alpha_l \nabla R\left(\widetilde{U}^{(i)}_K\right) \\
& \qquad \widetilde{U}^{(i)}_K \leftarrow H^{(i)}H^{(i)^T} \widetilde{U}^{(i)}_K
\end{aligned}
\end{equation}
Using \eqref{eq:subGraLocal}, the pseudo-inverse of the local covariance is given by:
\begin{equation}
\label{eq:new_local_cov}
\widetilde{\Sigma}_{i}^{^{\dagger}} =\widetilde{U}_{d}^{\left(i\right)}{\Lambda_{d}^{(i)}}^{-1}\widetilde{U}_{d}^{\left(i\right)^{T}}
\end{equation}
By replacing $\widehat{\Sigma}_i$ with $\widetilde{\Sigma}_i$ we rewrite \eqref{eq:d_LM} and obtain the informed local Mahalanobis distance:
\begin{equation}
\label{eq:d_inf}
d_{ILM}^2\left(c_{1},c_{2}\right) =\frac{1}{2}\left(c_{1}-c_{2}\right)^{T}\left(\widetilde{\Sigma}_{1}^{^{\dagger}}+\widetilde{\Sigma}_{2}^{^{\dagger}}\right)\left(c_{1}-c_{2}\right)
\end{equation}
The construction procedure of $d_{ILM}$ is summarized in Algorithm \ref{alg:Local Mahalanobis-distance-informed}. 

\begin{algorithm}
	\textbf{input}: Data matrix $\mathbf{D}\in\mathbb{R}^{m\times n}$
	\begin{enumerate}
		\item For each columns $c_{i}$ compute the local inverse covariance:
		\begin{enumerate}
			\item Create the matrix $\mathbf{D}^{\left(i\right)}$ which consists of $N$ nearest neighbors
			of $c_{i}$ in its columns
			\item Apply k-means clustering to the rows of $\mathbf{D}^{\left(i\right)}$, and
			get $H^{\left(i\right)}$
			\item Find $d\leq K$ principal directions of $\mathbf{D}^{\left(i\right)}$, $U_{d}^{\left(i\right)}=\left[u_{1}^{\left(i\right)}\,u_{2}^{\left(i\right)}\,...\,u_{d}^{\left(i\right)}\right]$, by taking the
			$d$ eigenvectors of $\Sigma_{D^{(i)}}=\frac{1}{N-1}\mathbf{D}^{\left(i\right)}\mathbf{D}^{\left(i\right)^T}$ with the largest eigenvalues $\left\{ \lambda_{j}^{\left(i\right)}\right\} _{j=1}^{d}$
			\item Compute new principal directions: \begin{equation}
			\begin{aligned}
			& \text{for $l = 1,2,3,...$} \\
			& \qquad \widetilde{U}^{(i)}_d \leftarrow \widetilde{U}^{(i)}_d - \alpha_l \nabla R\left(\widetilde{U}^{(i)}_d\right) \\ \nonumber
			& \qquad \widetilde{U}^{(i)}_d \leftarrow H^{(i)}H^{(i)^T} \widetilde{U}^{(i)}_d
			\end{aligned}
			\end{equation}
			where:
			\begin{equation}
			\nabla R\left(\widetilde{U}^{(i)}_d\right) =-2\left(\left(I-\widetilde{U}^{(i)}_d \widetilde{U}^{(i)^T}_d\right)\Sigma_{D^{(i)}} + \Sigma_{D^{(i)}} \left(I-\widetilde{U}^{(i)}_d \widetilde{U}^{(i)^T}_d\right)\right)\widetilde{U}^{(i)}_d \nonumber
			\end{equation}
			\item Normalize the principal directions $\widetilde{U}_{d}^{\left(i\right)}$
			\item Compute the local pseudo-inverse covariance of the point $c_{i}$:
			\begin{align*}
			\widetilde{\Sigma}_{i}^{^{-1}} & =\widetilde{U}_{d}^{\left(i\right)}{\Lambda_{d}^{(i)}}^{-1}\widetilde{U}_{d}^{\left(i\right)^{T}}\\
			{\Lambda_{d}^{(i)}}^{-1} & =diag\left\{ \lambda_{1}^{\left(i\right)^{-1}},\lambda_{2}^{\left(i\right)^{-1}},...,\lambda_{d}^{\left(i\right)^{-1}}\right\} 
			\end{align*}
		\end{enumerate}
		\item The distance between two points:
		\[
		d_{ILM}^2\left(c_{1},c_{2}\right)={\frac{1}{2}\left(c_{1}-c_{2}\right)^{T}\left(\widetilde{\Sigma}_{1}^{^{-1}}+\widetilde{\Sigma}_{2}^{^{-1}}\right)\left(c_{1}-c_{2}\right)}
		\]
	\end{enumerate}
	\caption{Local Mahalanobis Distance Informed by Clustering \label{alg:Local Mahalanobis-distance-informed}}
\end{algorithm}
In the example presented in Section \ref{subsec:Global-exp} we have shown
an equivalence between the global Mahalanobis distance, $\hat{d}_M$, and the one
informed by clustering, $\tilde{d}_{M}$.
However where there is a constraint on the number
of columns, as required for locality in the local
Mahalanobis distance, the estimated pseudo-inverse of the local covariance is less accurate. 
For a covariance of rank $d$, estimated by $N$ columns, the estimation error is increased with $d/N$. 
However, these two parameters do not necessarily degrade the performance of the clustering
of the rows: k-means can be applied successfully to $m$ rows even for a large $m$ (or $d$), and regardless of the dimension of the rows $N.$ Once the clustering of the rows is computed, we utilize it to improve the estimation of the local pseudo-inverse covariance.
Therefore we expect our approach to be beneficial whenever the number
of columns available for the estimation of the high-rank pseudo-inverse covariance is limited, provided that the rows exhibit a typical clustering structure.
This will be demonstrated in the next section.

\section{Synthetic Toy Problem}
\label{sec:toy_exp}
\subsection{Description}

To examine our approach we generate synthetic data according to
the problem setting described previously: samples are in high dimensional
space, and coordinates have a distinct structure of clusters. 
Consider high-dimensional data stored in the matrix $\mathbf{D}$, where the columns and rows 
represent, for example, subjects and their properties, respectively. 
We assume that the columns (e.g., subjects) can belong to two types (e.g., two types of population), namely, TypeA and TypeB, 
with the same probability $p=0.5$.
The rows (e.g, properties) are divided into $K$ clusters (e.g., demographic properties, physiological properties, etc.), $G_k, k=1,\ldots,K$, such that each two properties are correlated if and only if they belong to the same cluster. 
Given the type of the column and the cluster of the row, the
distribution of the elements of the data is given by:
\begin{align}
\label{eq:toy_dist}
D_{ij}|\,\,i  \in G_{k},\,j\in \mathrm{TypeA} &\sim\mathcal{N}\left(\mu_{kA},\sigma^{2}\right)\nonumber \\ 
D_{ij}|\,\,i \in G_{k},\,j\in \mathrm{TypeB} &\sim\mathcal{N}\left(\mu_{kB},\sigma^{2}\right)
\end{align}
The covariance matrix of the columns is given by a block diagonal matrix:
\begin{eqnarray*}
	C & = & \left(\begin{array}{cccc}
		\Gamma_{1} & 0 & \cdots & 0\\
		0 & \Gamma_{2} &  & \vdots\\
		\vdots &  & \ddots & 0\\
		0 & \cdots & 0 & \Gamma_{K}
	\end{array}\right) \in \mathbb{R}^{m\times m}
\end{eqnarray*}
where $\Gamma_{k}, k=1,\ldots,K$ contains the variance $\sigma^{2}$ and the covariance $\rho$ elements between the properties
in cluster $G_{k}$:
\[
\Gamma_{k}=\left(\begin{array}{cccc}
\sigma^{2} & \rho & \cdots & \rho\\
\rho & \sigma^{2} &  & \vdots\\
\vdots &  & \ddots & \rho\\
\rho & \cdots & \rho & \sigma^{2}
\end{array}\right) \in \mathbb{R}^{m_{k}\times m_{k}}
\]
Here $m_{k}$ represents the size of cluster $G_{k}$. 
Let $\mu_{A}$ and $\mu_{B}$ be the conditional expected value of a column given its type,
i.e. $\mu_{A}=E\left[c_{j}\,|\,j\in \mathrm{TypeA}\right]$
and $\mu_{B}=E\left[c_{j}\,|\,j\in \mathrm{TypeB}\right]$. Using the law of
total expectation, the covariance matrix of all columns is given
by:
\begin{equation}
\label{eq:toy_sigma}
\Sigma=C+\frac{1}{4}\left(\mu_{A}-\mu_{B}\right)\left(\mu_{A}-\mu_{B}\right)^{T}
\end{equation}
Figure \ref{fig:quest_data_unordered} presents a data set of $100$ samples (columns) generated
according to the above model. For illustration purposes, in Fig. \ref{fig:quest_data_ordered}  the columns and rows
are ordered according to their type and cluster, respectively.
Note that the proposed approach does not take into account the specific order of the columns and rows of the given data set. 


\begin{figure}
	\begin{centering}
		\subfloat[\label{fig:quest_data_unordered}]{\centering{}\includegraphics[width=0.26\paperwidth,height=0.275\paperheight]{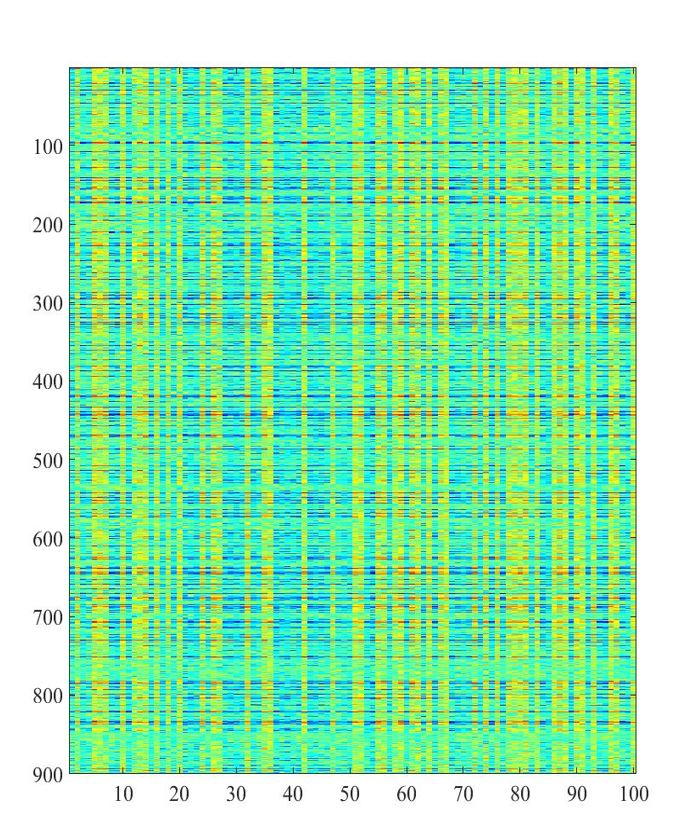}}
		\subfloat[\label{fig:quest_data_ordered}]{
			\centering{}\includegraphics[width=0.52\paperwidth,height=0.27\paperheight]{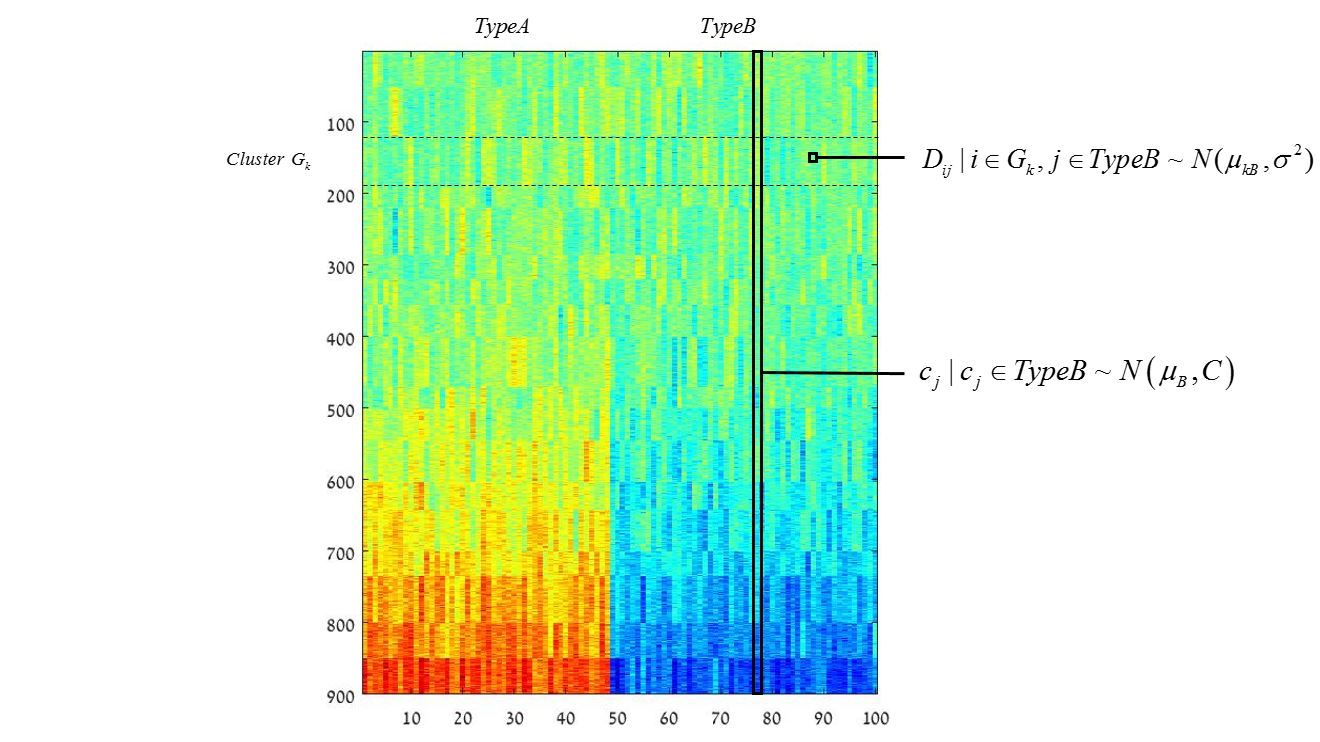}}
		\par\end{centering}
	\centering{}\caption{Data set of $n=100$ samples generated according to the model described in Section \ref{sec:toy_exp}. Entry $D_{ij}$ represents the $i$th property of the $j$th subject. The sample $c_{j}$ represents all the properties of the $j$th subject.(a) The matrix $\mathbf{D}$ where the order of columns and rows is random. (b) Columns and rows are organized according to clusters and types respectively. \label{fig:quest_data}}
\end{figure}

\subsection{Results}
\begin{figure}[t]
	\begin{centering}
		\subfloat[]{
			\centering{}\includegraphics[width=0.35\paperwidth,height=0.2\paperheight]{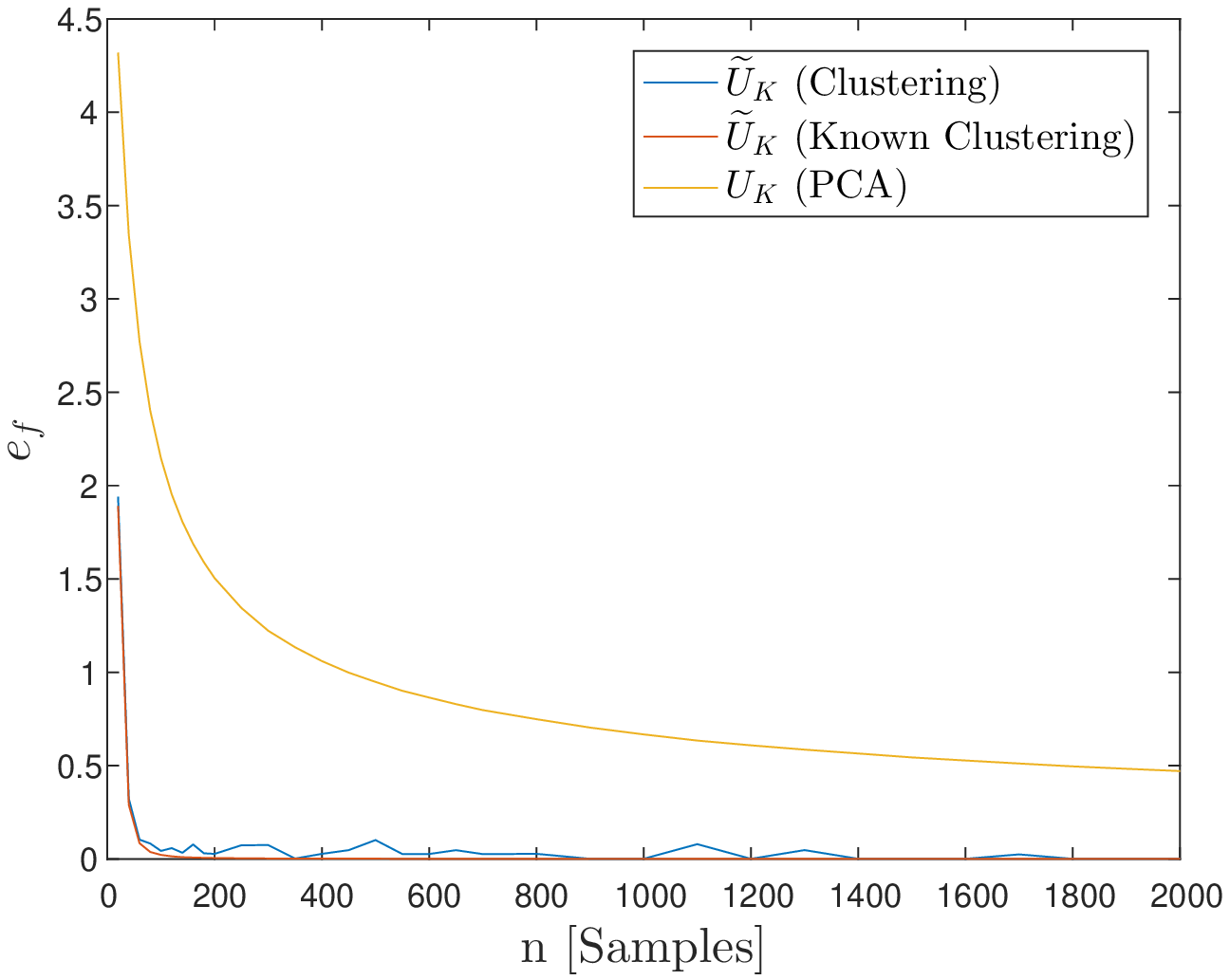}}
		\subfloat[]{
			\centering{}\includegraphics[width=0.35\paperwidth,height=0.2\paperheight]{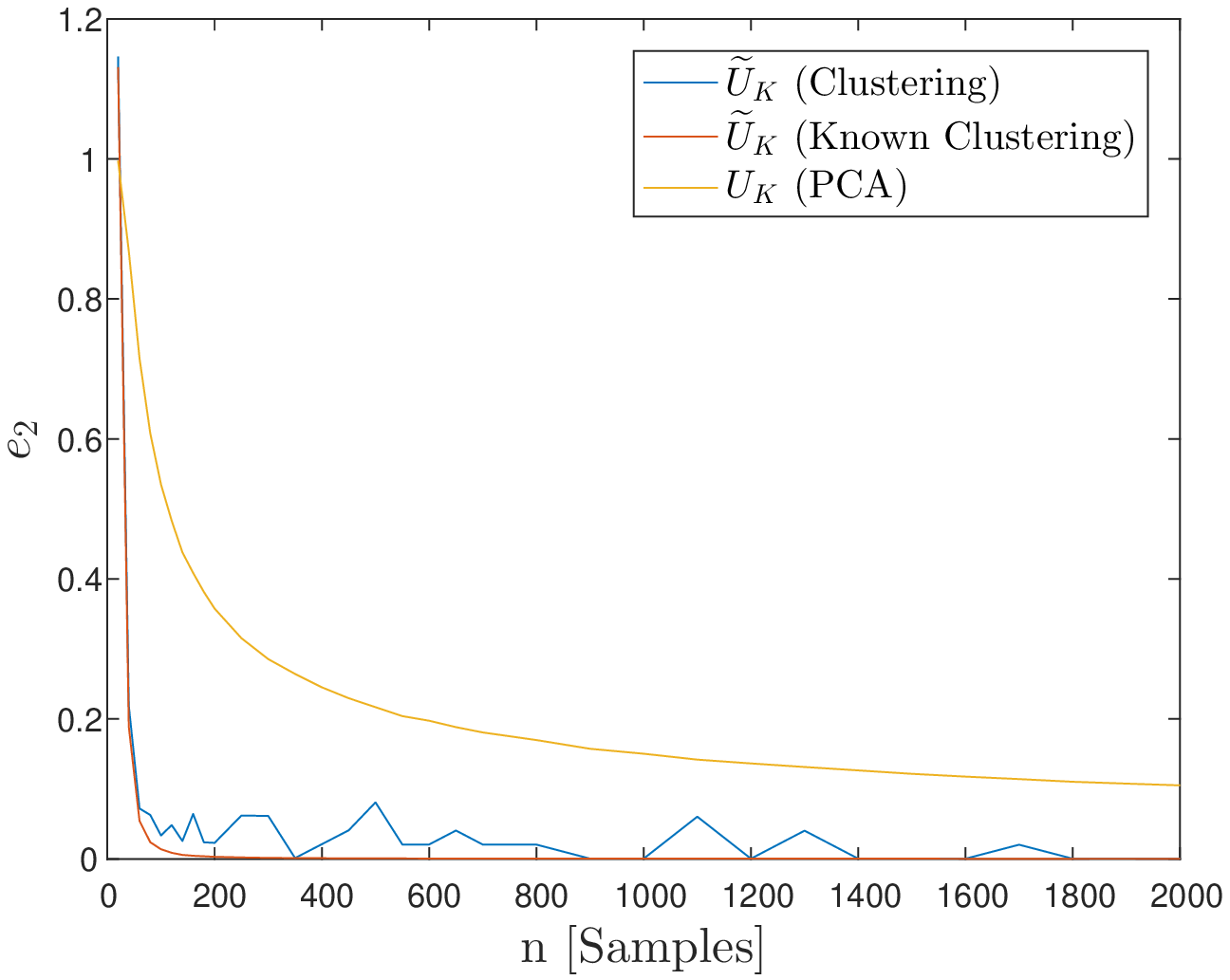}}
		\par\end{centering}
	\caption{Principal directions estimation error as a function of the number
		of subjects $n$, computed with: (a) Forbenius norm $\|\cdot\|_{F}$, (b) 2-norm $\|\cdot\|_{2}$. In yellow line the \ac{PCA} error. In blue and red lines
		the error of the proposed principal directions $\widetilde{U}_{K}$ using
		k-means and a known clustering solution respectively.} \label{fig:quest_err}
\end{figure}
We compare between the principal directions obtained by our method in \eqref{eq:subGrad} and the principal directions obtained by \ac{PCA}. To isolate the error obtained solely by the clustering assumption, we also compute the principal directions using a known clustering solution. 

Denote the estimated principal directions by $\widehat{U}_{K}$.
The objective measure used to evaluate the obtained principal directions $\widehat{U}_{K}$ is
the empirical mean error with respect to the first $K$ eigenvectors $\bar{U}_{K}$
of the true covariance $\Sigma$ calculated analytically from the model \eqref{eq:toy_sigma}. The empirical mean error is computed with the Forbenius norm:
\[
e_f=\langle \|\bar{U}_{K}\bar{U}_{K}^{T}-\widehat{U}_{K}\widehat{U}_{K}^{T}\|_F\rangle
\]
and with the 2-norm:
\[
e_2=\langle \|\bar{U}_{K}\bar{U}_{K}^{T}-\widehat{U}_{K}\widehat{U}_{K}^{T}\|_2\rangle
\]
where $\langle \cdot \rangle$ is the empirical mean over the realizations of the entire data set.
Note that this criterion is invariant to any orthogonal transformation (rotation) of $\bar{U}_{K}$. 
We note that in Section \ref{sec:Global-Mahalanobis} the number of principal directions is the number of the k-means clusters minus one. Here, to complete the number of principal directions to the rank $K$ of the true covariance $\Sigma$, we compute the $K$th principal direction, $\tilde{u}_K$, in the same way as in \eqref{eq:subGrad}.

We consider data with $m=900$ rows, divided into $K=18$ clusters. 
For different sizes of $n$, we generate $n$ samples (columns) according to \eqref{eq:toy_dist} and repeat the experiment $50$ times.
The error averaged over the experiments is presented in Fig. \ref{fig:quest_err} as a function of $n$. 
We observe that the error of $\widetilde{U}_{K}$ based on clustering is lower than
the error of PCA, especially for small values of $n$. While a small amount of columns results in a significant estimation error of the principal directions, rows can be clustered more easily. By exploiting these clusters, the proposed principal directions $\widetilde{U}_{K}$ indeed yields a better estimation.
As $n$ increases, the difference between the errors obtained by our method in \eqref{eq:subGrad} and by \ac{PCA} decreases.
This happens for two reasons. First, \ac{PCA}-based estimation improves when more samples are available for the sample covariance estimation $\Sigma_D$. Second, since the dimension of the rows $n$ increases, finding the
optimal solution of k-means clustering becomes more difficult. 
Conversely, we observe that for values of $n$ smaller than $60$, the known clustering has negligible contribution compared to the estimated clustering (since the dimension of the rows is small).
We note that errors computed with different norms (e.g. $\|\cdot\|_{1}$)
yield comparable results and trends. 

As was discussed in Section \ref{sec:Global-Mahalanobis}, an intuitive
explanation for the definition of $\widetilde{U}_{K}$ is the averaging
of the principal directions obtained by \ac{PCA} according to the
clustering of the rows found by k-means. 
To illustrate this property, we present
the true covariance matrix of the columns (ordered according to the clustering of the rows) in Fig. \ref{fig:quest_real_cov} and the estimated covariance matrices obtained by \ac{PCA} and by our method in Fig. \ref{fig:quest_pca_cov} and Fig. \ref{fig:quest_clust_cov}, respectively.

\begin{figure}[H]
	
	\begin{centering}
		\subfloat[\label{fig:quest_real_cov}]{\centering{}\includegraphics[width=0.35\paperwidth,height=0.2\paperheight]{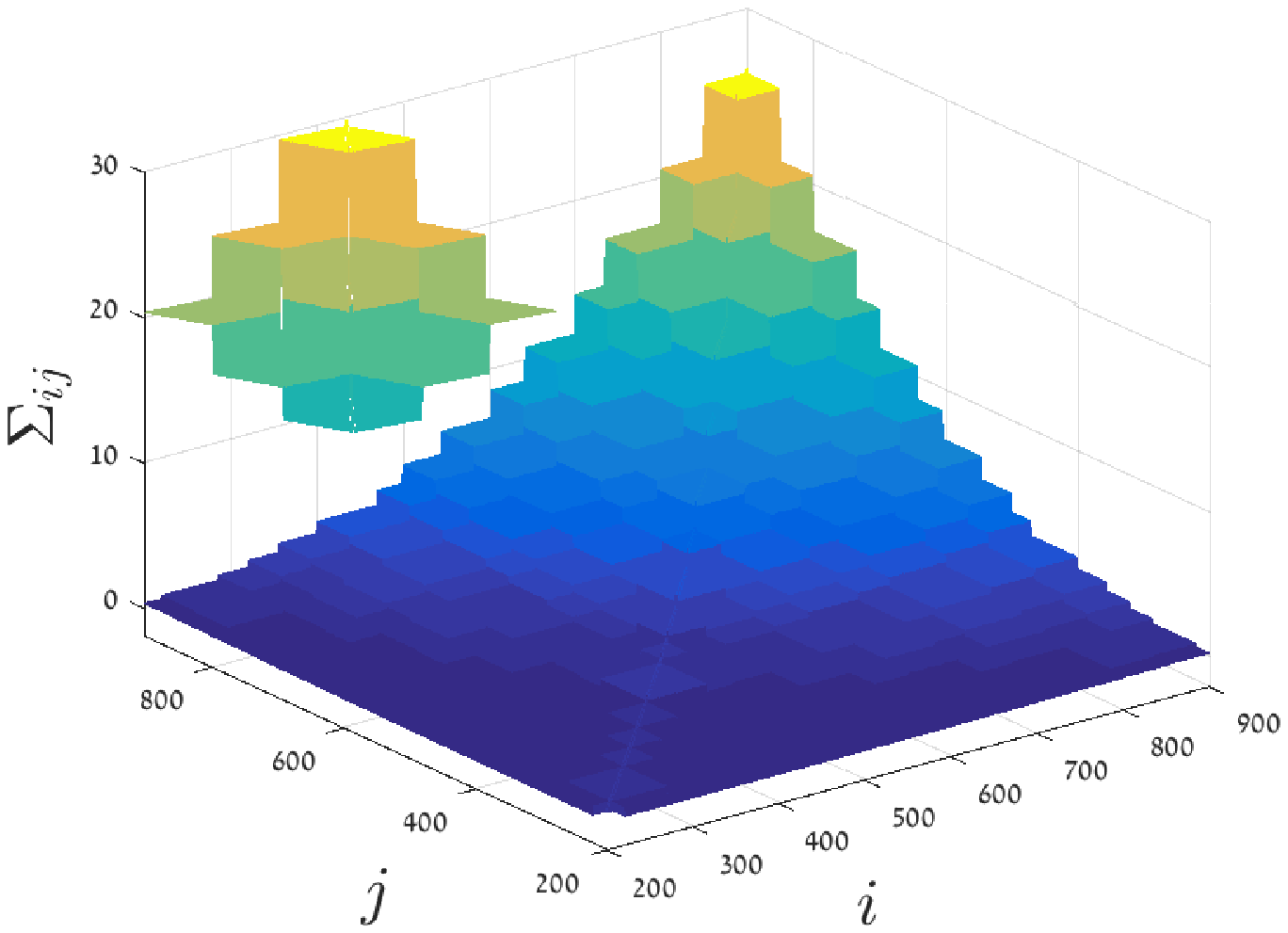}}
		\subfloat[\label{fig:quest_pca_cov}]{\centering{}\includegraphics[width=0.35\paperwidth,height=0.2\paperheight]{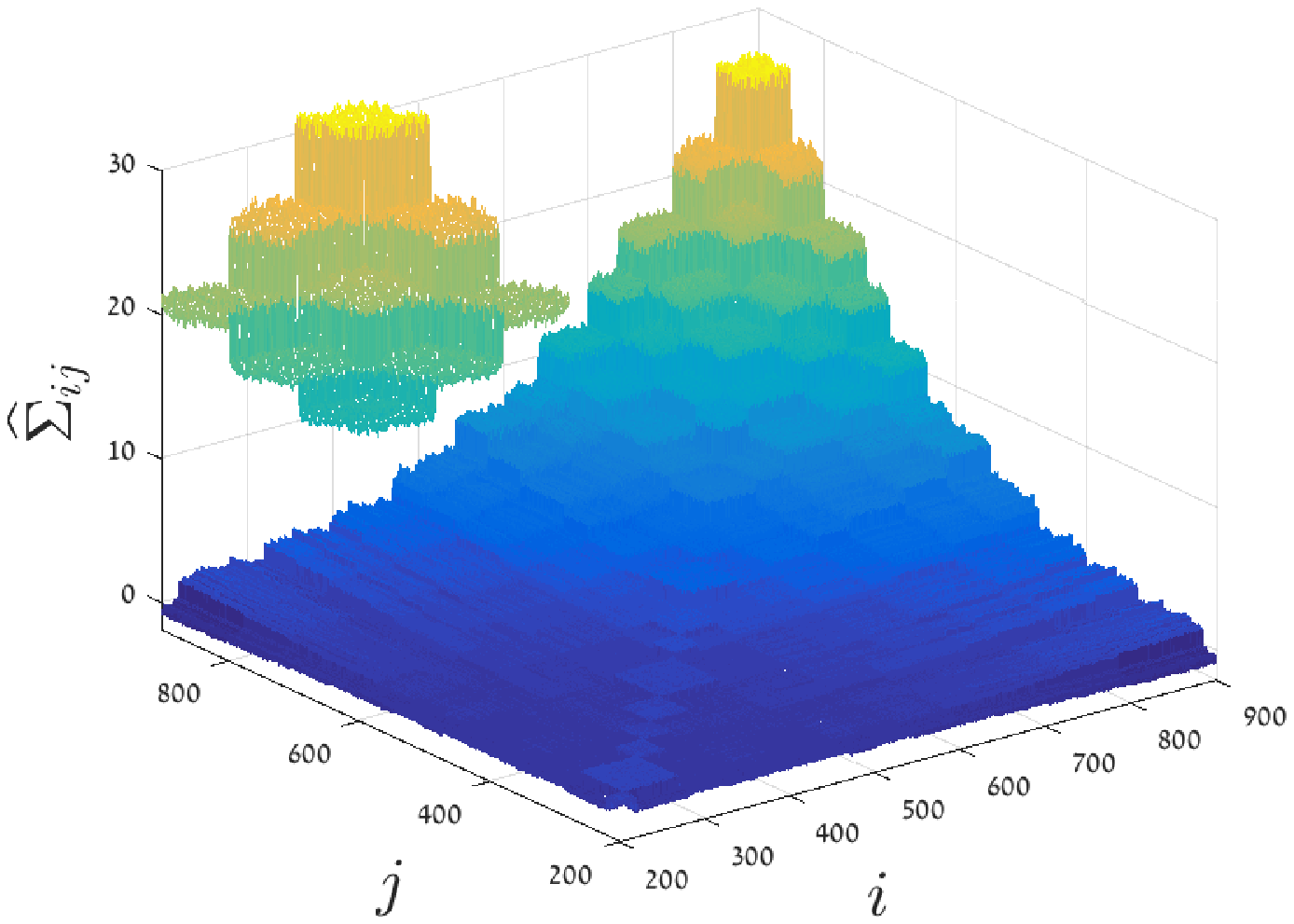}}
		\par\end{centering}
	\centering{}\subfloat[\label{fig:quest_clust_cov}]{\centering{}\includegraphics[width=0.35\paperwidth,height=0.2\paperheight]{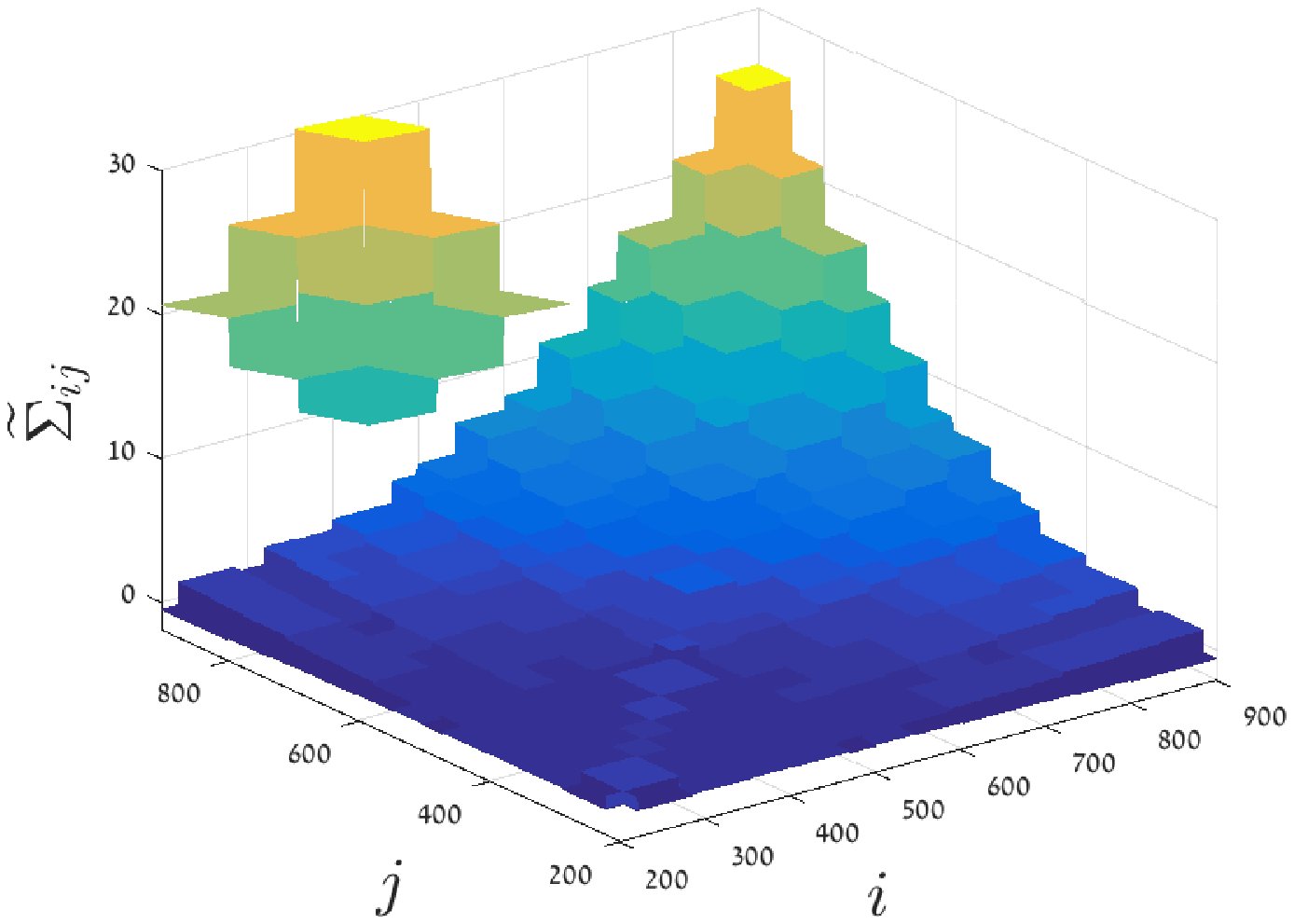}}
	\subfloat[\label{fig:quest_covariance_i900}]{\centering{}\includegraphics[width=0.35\paperwidth,height=0.2\paperheight]{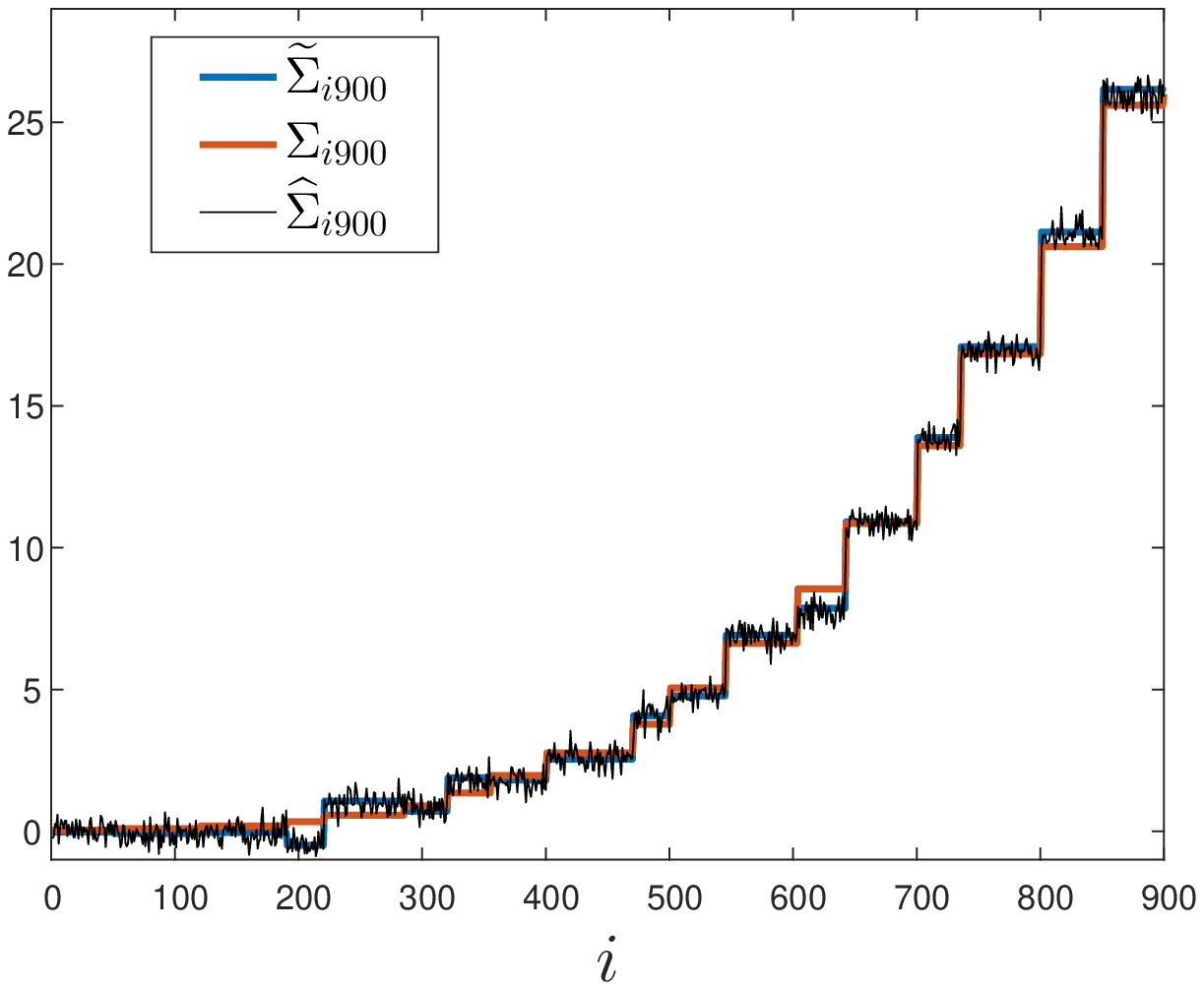}}
	
	\caption{True and estimated covariance matrices of subjects with zoom on properties 750-900: (a) the true
		covariance $\Sigma$, (b) the covariance $\widehat{\Sigma}$ obtained by \ac{PCA},
		(c) the covariance $\widetilde{\Sigma}$ estimated by the proposed approach. (d) Cross section of the covariance matrices in (a), (b) and (c) at property j=900.
		Covariance matrices were estimated using $n=100$ samples. \label{fig:illustration_cov} }
\end{figure}
As can be seen, the true covariance takes the shape of ``stairs'', which is induced by the model.
The estimation error of the covariance obtained by \ac{PCA} (Fig. \ref{fig:quest_pca_cov})
is manifested as noise mounted on each stair.
The contribution of the clustering is evident in Fig. \ref{fig:quest_clust_cov}, where we observe that when the clusters are taken into account, each stair is denoised by averaging all the entries belonging to this stair, resulting
in a structure that visually more similar to the true covariance depicted in Fig. \ref{fig:quest_real_cov}. The smoothing effect of each stair is illustrated also by comparing one of the rows of the covariance (the $900$th property) as presented in Fig. \ref{fig:quest_covariance_i900}. 
\begin{figure}
	\begin{centering}
		\subfloat[\label{fig:quest_embd_pca}]{\centering{}\includegraphics[width=0.35\paperwidth,height=0.2\paperheight]{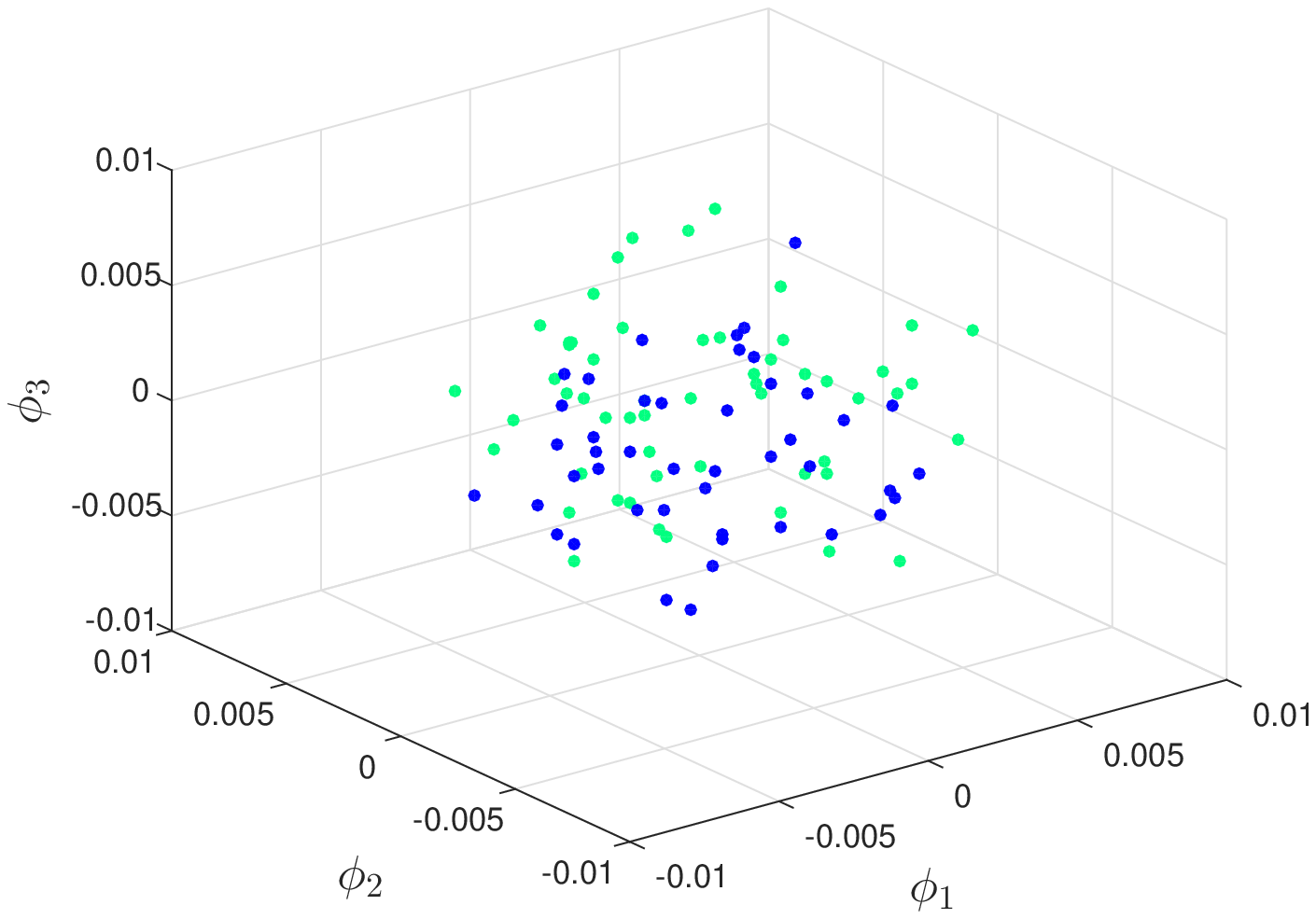}}\subfloat[\label{fig:quest_embd_clust}]{\begin{centering}
				\includegraphics[width=0.35\paperwidth,height=0.2\paperheight]{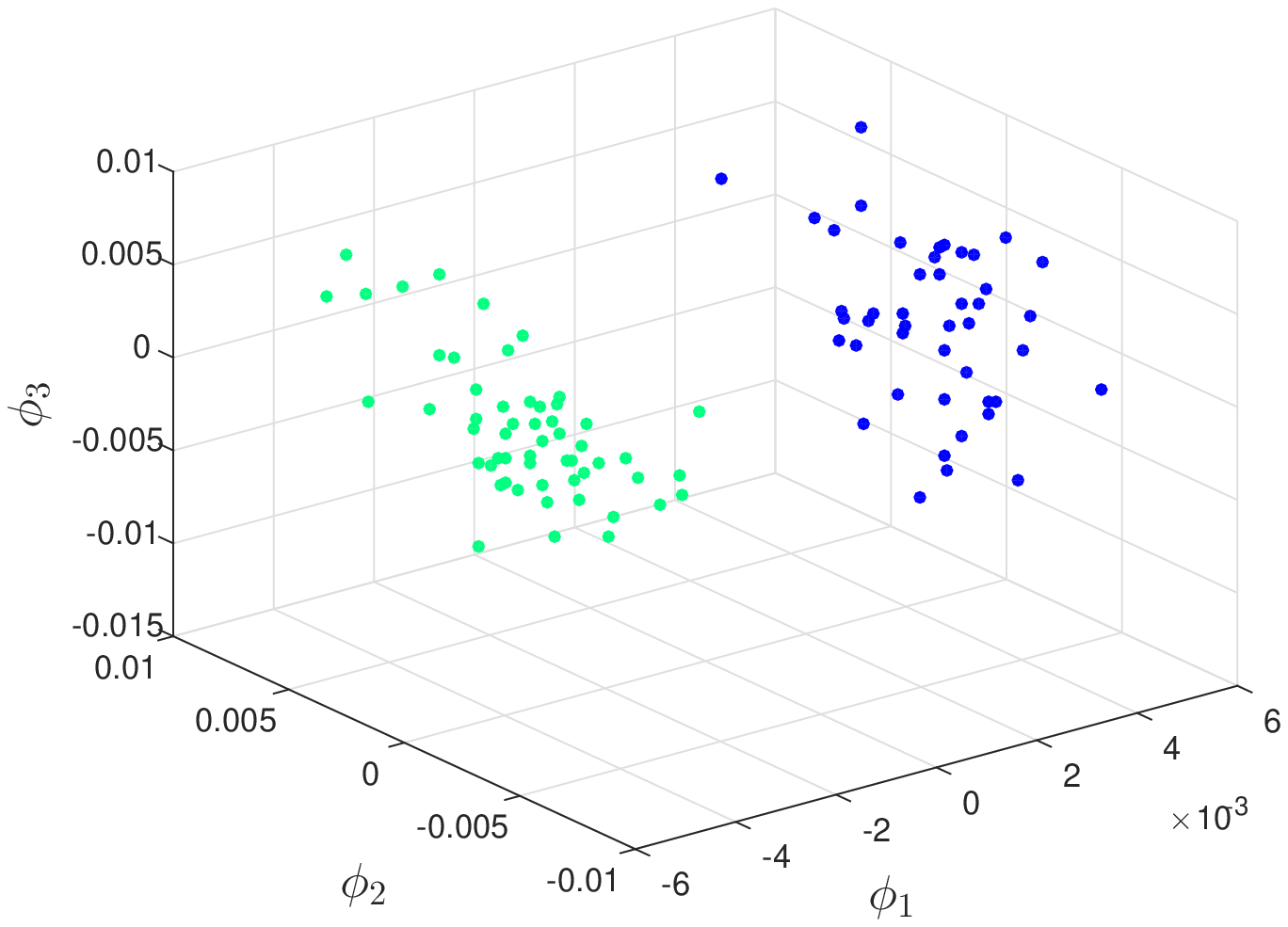}
				\par\end{centering}
		}
		\par\end{centering}
	\centering{}\caption{Embedding of the columns colored according to their type (TypeA, TypeB). (a) embedding
		obtained by the conventional Mahalanobis distance $\hat{d}_{M}.$ (b) embedding
		obtained by the proposed distance $\tilde{d}_{M}.$ \label{fig:quest_embd}}
\end{figure}

Recall that the covariance matrices $\widehat{\Sigma}$ and $\widetilde{\Sigma}$ (which are estimated by \ac{PCA} \eqref{eq:conventional_cov_rank_k} and our approach \eqref{eq:new_cov}) can be used to define the respective Mahalanobis distances $\hat{d}_{M}$ and $\tilde{d}_{M}$. To evaluate the benefit of our method in terms of the metric, we compute the diffusion maps of $n=100$ columns as described in Section \ref{subsec:Global-exp} twice separately based on $\hat{d}_{M}$ and $\tilde{d}_{M}$.

The embedding (of the columns) obtained by each of the distances is presented in Fig. \ref{fig:quest_embd}, where each embedded column is colored according to its type. 
It can be clearly observed in Fig. \ref{fig:quest_embd_clust} that when the diffusion maps is based on the proposed distance $\tilde{d}_{M}$, columns in the embedded space are well separated according to their type. Conversely, as presented in Fig. \ref{fig:quest_embd_pca},
when using $\hat{d}_{M}$ the two point clouds representing the two types are mixed.

Note that the proposed metric $\tilde{d}_{M}$ can be incorporated not only into diffusion maps algorithm, but also into other dimensionality reduction algorithms, such as t-SNE \cite{maaten2008visualizing}.

\section{Application to Gene Expression Data}
\label{sec:gene}

To demonstrate the applicability of our approach we apply it to real data, and more specifically, to gene expression data. The technology of DNA microarray allows to measure the expression levels for tens of thousands of genes simultaneously. This massive genomic data, which contains gene profile of different subjects, is used for a variety of analysis tasks such as: predicting gene function classes, cancer classification, predicting putative regulatory signals in the genome sequences \cite{brazma2000gene}, identifying prognostic gene signatures \cite{Ma2015,shedden2008gene}, unsupervised discrimination of cell types \cite{kluger2004lineage}, etc. 

We apply our method to gene expression data set for the purpose of analysing lung cancer subjects. Using the proposed method we will show that we can accomplish an analysis task, which was previously approached by several prognostic models \cite{Ma2015,shedden2008gene}: clustering and ordering subjects with respect to their (unknown) survival rate. Here, this will be accomplished through the construction of our informed metric between the subjects. 

The gene profile of a subject in a typical data set may contains tens of thousands of genes. Therefore, the application of a metric such as the local Mahalanobis distance \cite{Singer2008}, which is based on a na\"{i}ve estimation of the local covariance from a limited number of high dimensional samples (e.g., the sample covariance which is typically used in \ac{PCA} applications) often yields   poor performance. 
We will show that using the metric proposed in Algorithm \ref{alg:Local Mahalanobis-distance-informed}, which is informed by clusters of genes, results in improved performance.

The setting of gene expression data can be formulated similarly to the problem setting considered in this paper. Consider a set of $n$ subjects for which a gene expression profiles composed of $m$ genes were collected. The data is organized in a matrix $\mathbf{D} \in \mathbb{R}^{m\times n}$ corresponding to the problem setting defined in Section \ref{sec:intro}. 

We analyze a data set termed CAN/DF (Dana-Farber Cancer Institute) and taken from \cite{shedden2008gene} consisting of gene expression profiles collected from lung adenocarcinomas (cancer subjects), which was previously analyzed in Shedden et al \cite{shedden2008gene}.
The data set consists of $82$ subjects, for which profiles of 22,283 gene expressions were collected.
Following common practice, we use only genes with high variance; specifically, we take the $200$ genes with the highest variance. 
The data is organized in a matrix $\mathbf{D} \in \mathbb{R}^{200\times82}$, such that the $i$th column corresponds to the
$i$th subject, and the $j$th row corresponds to the $j$th gene. 
In order to illustrate the clustering structure embodied in the genes, k-means is applied to the rows of $\mathbf{D}$, which are then
ordered according to the clusters. 
The heat map of the matrix $\mathbf{D}$ is presented in Fig. \ref{fig:genes_data_var} and the reordered matrix is presented in Fig. \ref{fig:genes_data_clust}. Indeed, in Fig. \ref{fig:genes_data_clust}, we can observe a clear structure of gene clusters.
\begin{figure}
	\begin{centering}
		\subfloat[\label{fig:genes_data_var}]{\centering{}\includegraphics[width=0.35\paperwidth,height=0.2\paperheight]{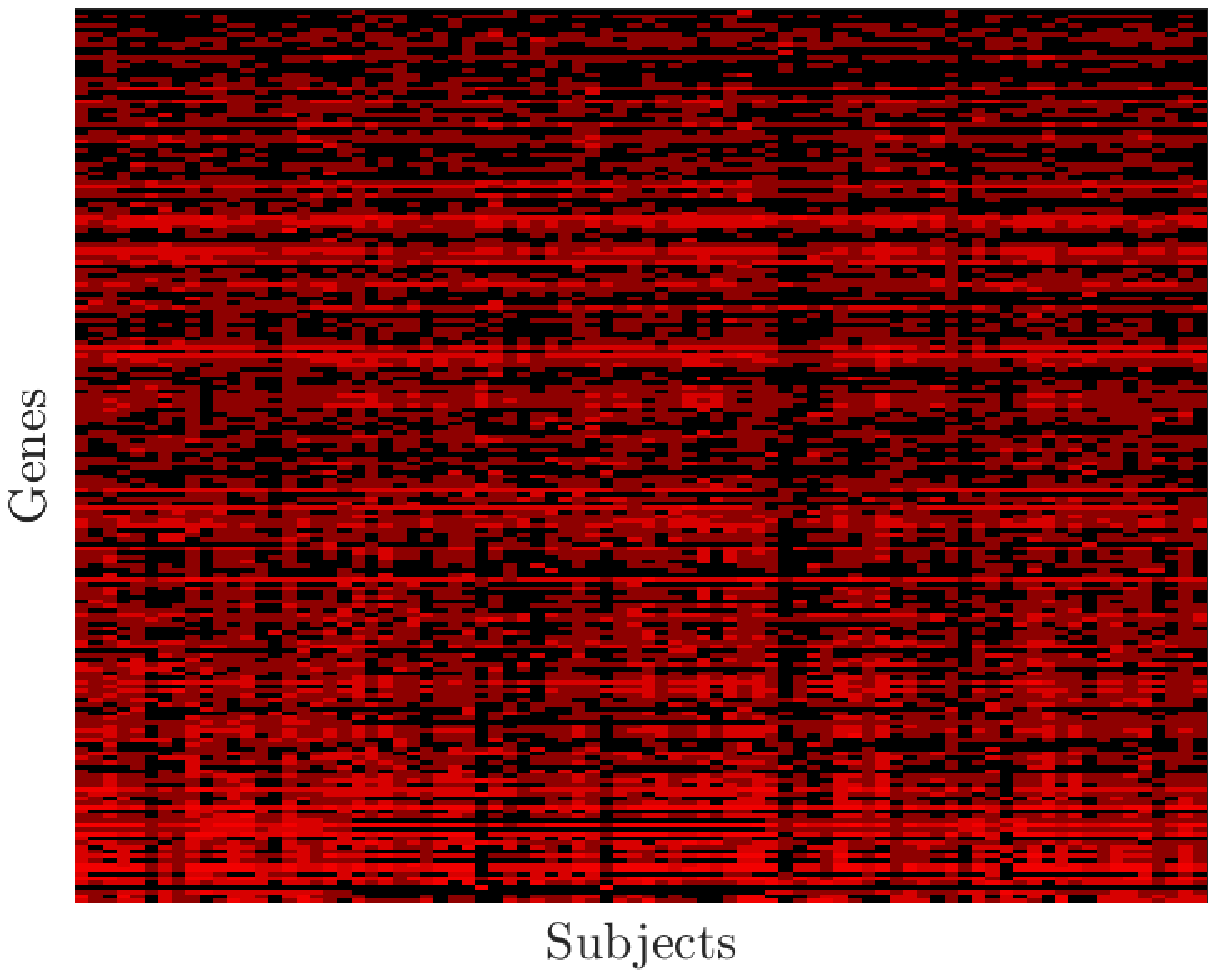}}\subfloat[\label{fig:genes_data_clust}]{\begin{centering}
				\includegraphics[width=0.35\paperwidth,height=0.2\paperheight]{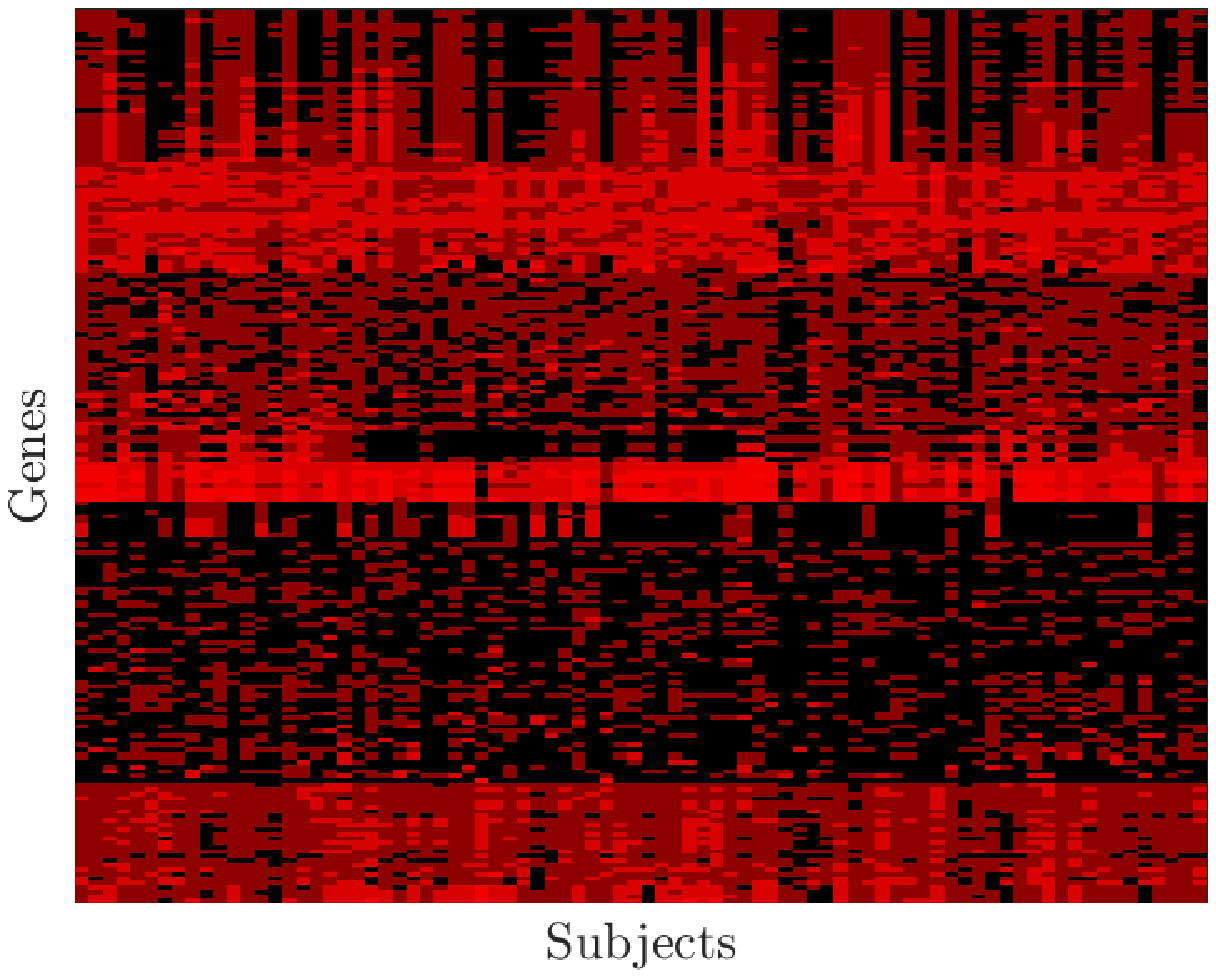}
				\par\end{centering}
		}
		\par\end{centering}
	\centering{}\caption{Heat maps of the expression of $m=200$ genes with the highest variance: (a) genes are ordered according to their variance, (b) genes are ordered according to $K=7$ clusters obtained by k-means.}
	\label{fig:genes_data}
\end{figure}

We calculate the local Mahalanobis distance between the columns of $\mathbf{D}$ in two ways. 
First, by applying local \ac{PCA} as in \cite{Singer2008} (based on the sample covariance), which will be denoted by $d_{LM}$ , and second, by Algorithm \ref{alg:Local Mahalanobis-distance-informed}, which will be denoted
by $d_{ILM}$. 
The number of principal directions used for the
computation of the local inverse covariance of each column can be estimated by the eigenvalues decay. 
Here, in both methods, we use $d=6$ principal directions.
The number of genes clusters used in Algorithm \ref{alg:Local Mahalanobis-distance-informed} is $K=7$. 
This particular value of $K$ was determined empirically, since it yields a good separation to clusters. 
We note that alternatively, $K$ can be set using standard algorithms designed for this purpose such as \cite{tibshirani2001estimating,pelleg2000x}. In addition, in the local \ac{PCA} we could use a different value of $K$ for each neighborhood. We set the neighborhood size to be $N=20$ samples. This choice is explained in the sequel.

The obtained metrics are objectively evaluated in the context of \emph{clustering the subjects (columns)}.
Specifically, we use the following algorithm.
(i) Compute the affinity matrix with a Gaussian kernel based on the competing metrics:
\[
W\left(c_{1},c_{2}\right)=e^{-\frac{d^{2}\left(c_{1},c_{2}\right)}{\varepsilon}}.
\]
Common practice is to set the kernel scale $\varepsilon$ to be of the order of the median of all pairwise distances
$d^{2}\left(c_{i},c_{j}\right)\,\,\forall i,j$. Here, we set
it to be twice the median. 
(ii) Build the diffusion map embedding $\left\{ \tilde{c}_{i}\right\} _{i=1}^{n}$ of the subjects (as described in Section \ref{sec:Global-Mahalanobis}).
(iii) Divide the subjects into two groups according to the sign of the principal component of the embedding  $\left\{ \tilde{c}_{i}\right\} _{i=1}^{n}$. In other words, we apply spectral clustering in the embedded space.

The quality of the clustering is measured with respect to the separation of the subjects to risk groups. Note that the clustering is completely unsupervised data-driven and it does not take into account the survival of subjects. Nevertheless, we get a good separation according to this hidden property. Figure \ref{fig:genes_km_plot} depicts the Kaplan-Meier survival plots obtained
by the two clustering algorithms (stemming from the two metrics). The survival plot of both groups, high and low risk, are presented in blue and red lines respectively.  
Figures \ref{fig:genes_kmPca} and \ref{fig:genes_kmClust} present the survival rate of both groups, where the clustering is obtained by the metric $d_{LM}$ and $d_{ILM}$ respectively.

It can be seen that the separation between the survival curves (blue and red) is more prominent when using the proposed metric $d_{ILM}$.
To objectively measure the degree of separation to two risk groups, we calculate the P-value by the log-rank test \cite{Cardillo2008logrank}
and average it over $20$ iterations. 
The P-value obtained by using the informed metric $d_{ILM}$ is $p_{ILM}=5.6\cdot10^{-3}$, 
and the P-value obtained by using the competing metric $d_{LM}$ is $p_{LM}=21.9\cdot10^{-3}$. 
This degree of separation is better, to the best of our knowledge, than the state of the art separation achieved by an \emph{unsupervised} algorithm SPARCoC \cite{Ma2015}; SPARCoC achieved $p=9.6\cdot10^{-3}$ for the same data set.
We note that SPARCoC was especially designed for molecular pattern discovery and cancer gene identification, whereas our algorithm is generic.
We also note that our results are comparable to other \emph{supervised} algorithms that were trained with two other data sets where the survival of each subject is known, as reported in \cite{shedden2008gene}.
\begin{figure}
	\begin{centering}
		\subfloat[\label{fig:genes_kmPca}]{\centering{}\includegraphics[width=0.35\paperwidth,height=0.2\paperheight]{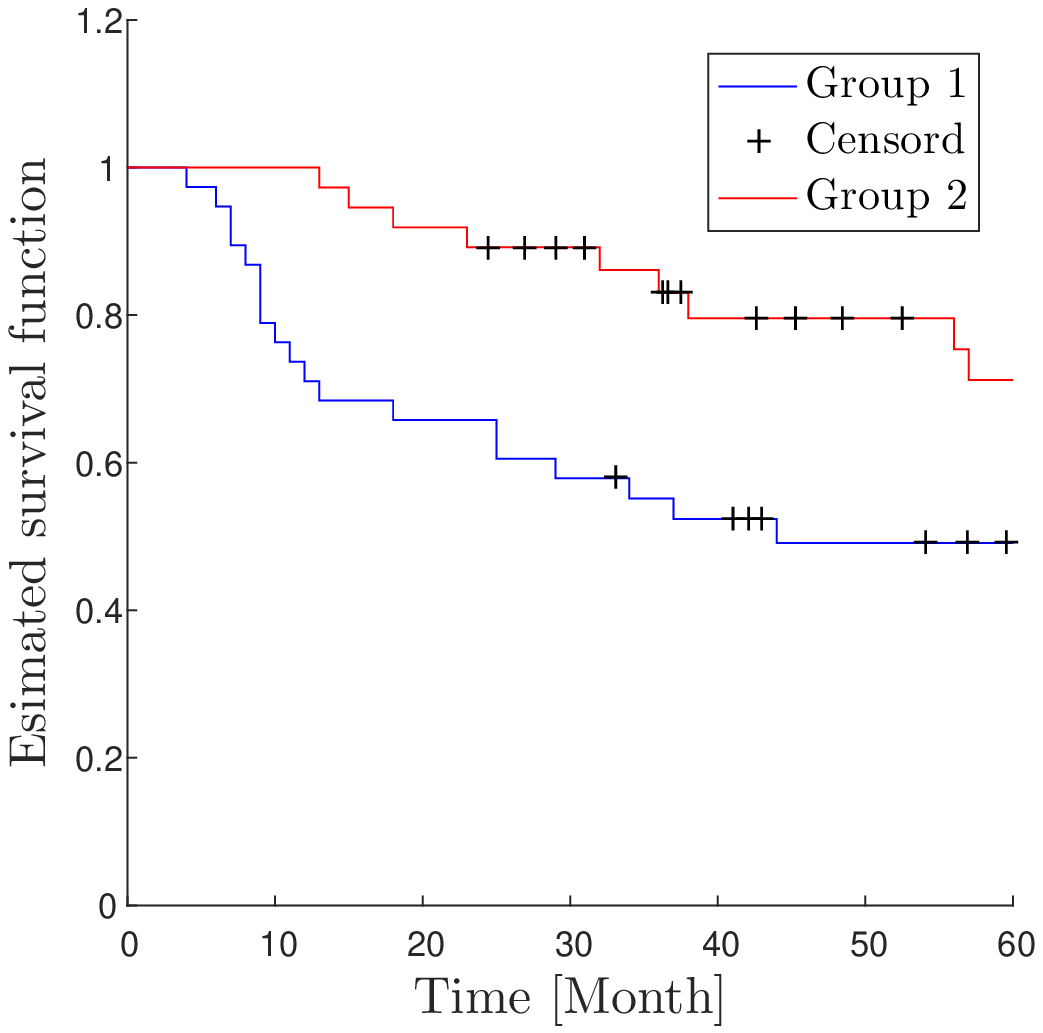}}
		\subfloat[\label{fig:genes_kmClust}]{\begin{centering}\includegraphics[width=0.35\paperwidth,height=0.2\paperheight]{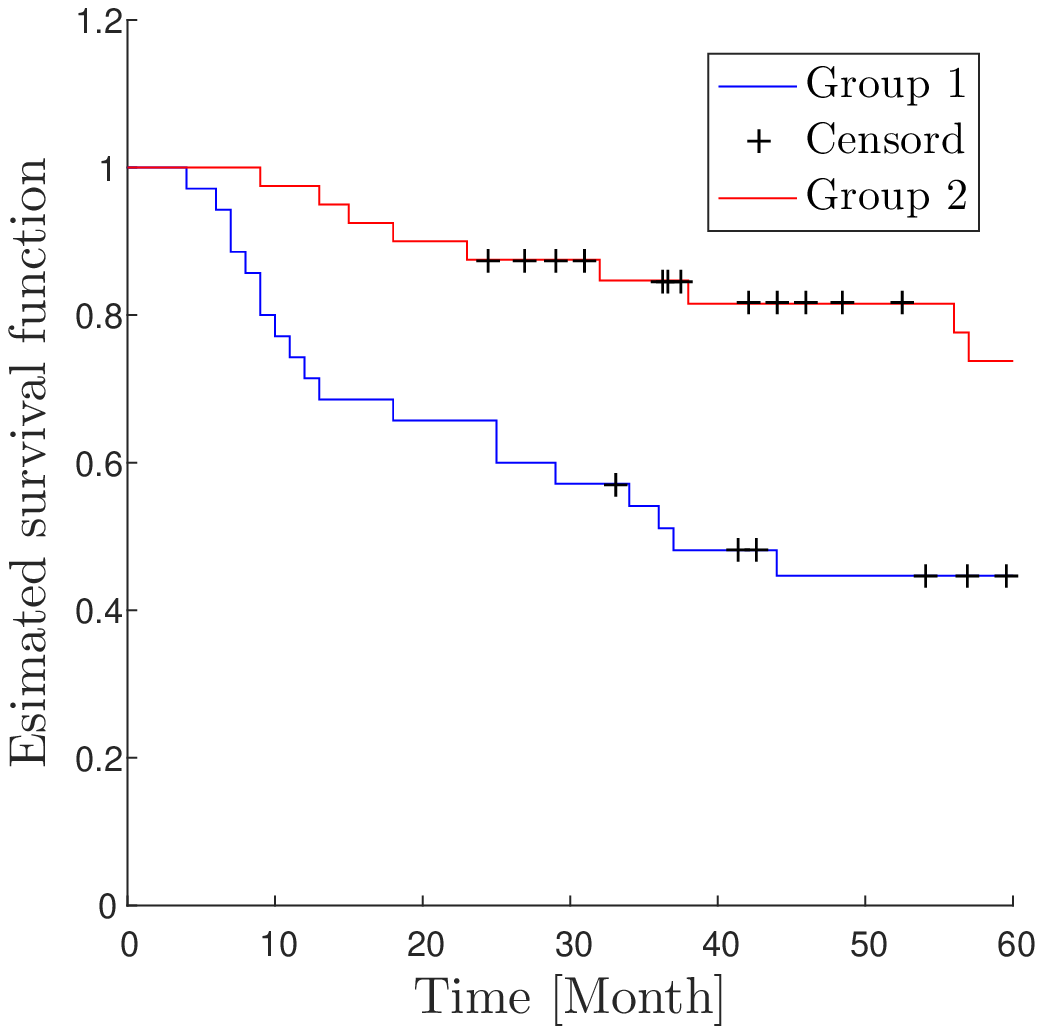}
				\par\end{centering}
		}
		\par\end{centering}
	\centering{}\caption{Kaplan-Meier survival plot of both risk groups. Subjects were divided into two risk groups using: (a) the metric $d_{LM}$, (b) the proposed metric $d_{ILM}$.}
	\label{fig:genes_km_plot}
\end{figure}

Our approach is designed specifically to address cases where
the amount of samples is small compared to the dimension of the samples.
In the computation of the local Mahalanobis distance, the sample set size is defined by the neighborhood size $N$ of each column. 
To demonstrate the benefits stemming from our method in such scenarios, Fig. \ref{fig:pval_vs_N} presents
the P-value obtained for different size of neighborhoods $N$. The red line represents the P-value $p_{LM}$ achieved by the metric $d_{LM}$. Since the P-value $p_{ILM}$ obtained by the metric $d_{ILM}$ depends on the initialization of k-means applied to each neighborhood, we repeat the calculation of $p_{ILM}$ 20 times. The mean is presented in blue line and the range of one \ac{STD} in gray.
We observe that the obtained P-value $p_{ILM}$ is smaller than $p_{LM}$ for small values of $N$, and larger for large values of $N$. The two curves approximately intersect at $N = 50$. Furthermore, the best P-value is achieved by $p_{ILM}$ when $N=34$. Yet, similar performance are attained for a relatively large range of $N$ values (ranging from $20$ to $35$), thereby implying on the robustness to the specific choice of neighborhood size. In addition, note the small \ac{STD} implying on robustness to the k-means initialization as well. 

\begin{figure}
	\centering{}\includegraphics[width=0.35\paperwidth,height=0.2\paperheight]{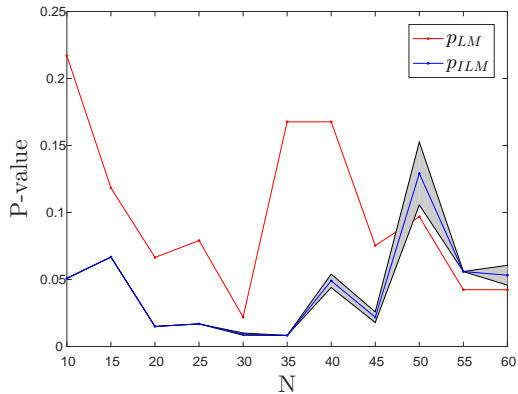}\caption{P-value calculated by log-rank test \cite{Cardillo2008logrank} as a function of the neighborhood size $N$. The P-value $p_{ILM}$ obtained by the proposed distance $d_{ILM}$ is computed for 20 iterations. The mean of $p_{ILM}$ is presented in blue line and the range of one STD in gray. The P-value $p_{LM}$ obtained by the distance $d_{LM}$ is presented in red line.}
	\label{fig:pval_vs_N}
\end{figure}
%
Since the advantage of our approach depends on the neighborhood
size $N$, we show how to find a range of good values of $N$, for the computation of $d_{ILM}$ from the data without using the knowledge of the survival rates.

As presented in Algorithm \ref{alg:Local Mahalanobis-distance-informed}, the matrix $\mathbf{D}^{(i)}$ contains
all the neighbors of $c_{i}$. The inverse local covariance $\Sigma_{i}^{^{-1}}$
is based on the clustering of the rows of $\mathbf{D}^{(i)}$, where the dimension of each row is $N$. 
We therefore use a method called gap statistic \cite{tibshirani2001estimating},
which defines a criterion $g\left(N,K\right)$ for clusters separability that also depends on the dimension $N$. 
The estimate of the range of good values for $N$ is set as the range where $g\left(N,K=k_{0}\right)$ stops increasing with $N$.

This estimation is based on the following intuition. Starting from $N=1$, increasing $N$, increases the separability of the rows of $\mathbf{D}^{(i)}$ to $k_0$ clusters. The indication of a good value of 
neighborhood size $N$
occurs, when the increase in $N$ (adding more columns to $\mathbf{D}^{(i)}$) no longer improves the separability, which results in a significant change in the slope of $g\left(N,K=k_{0}\right)$.
%
Figure \ref{fig:genes_gap} presents the gap statistics curve $g\left(N,7\right)$ as a function of the neighborhood size $N$.
As can be seen,
$g\left(N,7\right)$ indeed stops increasing 
where $N$ ranges from $20$ to $35$,
corresponding to the range of the
optimal neighborhood size 
found experimentally (Fig. \ref{fig:pval_vs_N}).
\begin{figure}
	\centering{}\includegraphics[width=0.35\paperwidth,height=0.2\paperheight]{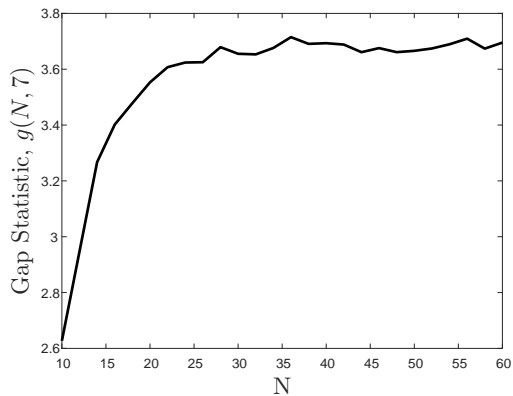}\caption{Gap statistic for $K=7$ clusters as a function of the neighborhood size $N$.}
	\label{fig:genes_gap}
\end{figure}

\section{Conclusions}
\label{sec:conclusions}
In this paper we present a method for computing a metric between high-dimensional data samples informed by the structure of the coordinates. 
The computation of this metric relies on the estimated (local or global) principal directions of the samples. 
We showed that by using clusters of coordinates we obtain a better estimation of the principal directions of the samples compared with the conventional \ac{PCA}. 
The superiority of the {\em informed metric} was illustrated on both synthetic and real data, where we demonstrated the construction of low dimensional representations, accurate clustering (of samples), and the recovery of hidden variables. The informed principal directions can be incorporated not only into the Mahalanobis distance, as presented in this work, but also to any technique which is based on global or local \ac{PCA}.

In future work, we will examine two important aspects stemming from the present work. The first direction is a generalization of the presented approach based on an optimization problem which will take into account the trade-off between the cost function of the conventional \ac{PCA} (maximum variance) and the clustering structure of the coordinates. Such extension should make our approach compatible also to cases were the clusters are less salient. The second direction involves the development of an iterative procedure in which the propose scheme is applied in alternating manner to the rows and the columns.  

\section*{Acknowledgment}

We thank for Regev Cohen and Eden Sassoon for helpful discussions and useful suggestions.


\section*{Appendix A}

Let $A$ be $m\times k$ matrix, where $m<k$, and $\Sigma_{x}$ be
a positive definite matrix. We will prove that the pseudo-inverse
of the matrix $\Sigma=A\Sigma_{x}A^{T}$ is given by:
\[
\Sigma^{\dagger}=A^{\dagger^{T}}\Sigma_{x}^{-1}A^{\dagger}.
\]
\\
The matrix $\Sigma^{\dagger}$ is the pseudo-inverse of $\Sigma$
if the following four properties are satisfied \cite{golub2012matrix}:

\emph{}\\
\emph{Property 1: }$\Sigma\Sigma^{\dagger}\Sigma=\Sigma$

\begin{proof}
	\emph{~
		\begin{eqnarray*}
			\Sigma\Sigma^{\dagger}\Sigma & = & \left(A\Sigma_{x}A^{T}\right)\left(A^{\dagger^{T}}\Sigma_{x}^{-1}A^{\dagger}\right)\left(A\Sigma_{x}A^{T}\right)\\
			& = & A\Sigma_{x}\left(A^{T}A^{\dagger^{T}}\right)\Sigma_{x}^{-1}\left(A^{\dagger}A\right)\Sigma_{x}A^{T}\\
			& = & A\Sigma_{x}\Sigma_{x}^{-1}\Sigma_{x}A^{T}\\
			& = & A\Sigma_{x}A^{T}\\
			& = & \Sigma
		\end{eqnarray*}
	}
\end{proof}

\emph{Property 2: }$\Sigma^{\dagger}\Sigma\Sigma^{\dagger}=\Sigma^{\dagger}$

\begin{proof}
	\emph{~
		\begin{eqnarray*}
			\Sigma^{\dagger}\Sigma\Sigma^{\dagger} & = & \left(A^{\dagger^{T}}\Sigma_{x}^{-1}A^{\dagger}\right)\left(A\Sigma_{x}A^{T}\right)\left(A^{\dagger^{T}}\Sigma_{x}^{-1}A^{\dagger}\right)\\
			& = & A^{\dagger^{T}}\Sigma_{x}^{-1}\left(A^{\dagger}A\right)\Sigma_{x}\left(A^{T}A^{\dagger^{T}}\right)\Sigma_{x}^{-1}A^{\dagger}\\
			& = & A^{\dagger^{T}}\Sigma_{x}^{-1}\Sigma_{x}\Sigma_{x}^{-1}A^{\dagger}\\
			& = & A^{\dagger^{T}}\Sigma_{x}^{-1}A^{\dagger}\\
			& = & \Sigma^{\dagger}
		\end{eqnarray*}
	}
\end{proof}

\emph{Property 3: }$\left(\Sigma\Sigma^{\dagger}\right)^{T}=\Sigma\Sigma^{\dagger}$

\begin{proof}
	\emph{~
		\begin{eqnarray*}
			\left(\Sigma\Sigma^{\dagger}\right)^{T} & = & \left(\left(A\Sigma_{x}A^{T}\right)\left(A^{\dagger^{T}}\Sigma_{x}^{-1}A^{\dagger}\right)\right)^{T}\\
			& = & \left(AA^{\dagger}\right)^{T}\\
			& = & AA^{\dagger}\\
			& = & \left(A\Sigma_{x}A^{T}\right)\left(A^{\dagger^{T}}\Sigma_{x}^{-1}A^{\dagger}\right)\\
			& = & \Sigma\Sigma^{\dagger}
		\end{eqnarray*}
	}
	
	The third equality holds because $A^{\dagger}$ is the pseudo-inverse
	of $A$, and therefore $AA^{\dagger}$ is symmetric. 
\end{proof}

\emph{Property 4: }$\left(\Sigma^{\dagger}\Sigma\right)^{T}=\Sigma^{\dagger}\Sigma$

\begin{proof}
	\emph{~
		\begin{eqnarray*}
			\left(\Sigma^{\dagger}\Sigma\right)^{T} & = & \left(\left(A^{\dagger^{T}}\Sigma_{x}^{-1}A^{\dagger}\right)\left(A\Sigma_{x}A^{T}\right)\right)^{T}\\
			& = & \left(A^{\dagger^{T}}A^{T}\right)^{T}\\
			& = & AA^{\dagger}\\
			& = & \left(AA^{\dagger}\right)^{T}\\
			& = & A^{\dagger^{T}}A^{T}\\
			& = & \left(A^{\dagger^{T}}\Sigma_{x}^{-1}A^{\dagger}\right)\left(A\Sigma_{x}A^{T}\right)\\
			& = & \Sigma^{\dagger}\Sigma
		\end{eqnarray*}
	}
	
	The fourth equality holds because $A^{\dagger}$ is the pseudo-inverse
	of $A$, and therefore $AA^{\dagger}$ is symmetric. 
\end{proof}

\end{document}